\PassOptionsToPackage{table}{xcolor}
\documentclass{article}
\usepackage{iclr2027_conference}
\usepackage{newtxtext}

\usepackage{amsmath,amssymb,amsthm}
\usepackage{mathtools}
\usepackage{bm}
\usepackage{booktabs}
\usepackage{array}
\usepackage{tabularx}
\usepackage{microtype}
\usepackage[table]{xcolor}
\usepackage{multirow}
\usepackage{graphicx}
\usepackage{algorithm}
\usepackage{algorithmic}
\usepackage{wrapfig}
\usepackage{enumitem}
\usepackage{threeparttable}
\usepackage{caption}
\usepackage{placeins}
\makeatletter
\def\section{\@startsection{section}{1}{\z@}{-1.4ex plus -0.4ex minus
  -.2ex}{0.9ex plus 0.2ex minus 0.2ex}{\large\sc\raggedright}}
\def\subsection{\@startsection{subsection}{2}{\z@}{-1.2ex plus -0.3ex
  minus -.2ex}{0.4ex plus .2ex}{\normalsize\sc\raggedright}}
\makeatother
\usepackage{tcolorbox}
\tcbuselibrary{breakable,skins}
\usepackage{longtable}
\usepackage{titletoc}
\usepackage{subcaption}
\usepackage{url}
\usepackage{hyperref}
\usepackage{cleveref}

\newtheorem{theorem}{Theorem}
\newtheorem{proposition}{Proposition}
\newtheorem{corollary}{Corollary}
\newtheorem{lemma}{Lemma}
\newtheorem{definition}{Definition}
\newtheorem{remark}{Remark}
\newtheorem{assumption}{Assumption}
\crefname{assumption}{Assumption}{Assumptions}
\Crefname{assumption}{Assumption}{Assumptions}

\newcommand{\cA}{\mathcal{A}}

\newcommand{\cD}{\mathcal{D}}
\newcommand{\cX}{\mathcal{X}}

\newcommand{\cG}{\mathcal{G}}
\newcommand{\phiTrue}{\phi^{*}}
\newcommand{\phiModel}{\phi_{\theta}}
\newcommand{\dtv}{D_{\mathrm{TV}}}
\newcommand{\dkl}{D_{\mathrm{KL}}}
\newcommand{\Keff}{K_{\mathrm{eff}}}
\newcommand{\E}{\mathbb{E}}

\newcommand{\argmin}{\operatorname{arg\,min}}

\newcommand{\cN}{\mathcal{N}}

\definecolor{bestbg}{HTML}{CFE0F3}
\definecolor{secondbg}{HTML}{FBE4D5}
\newcommand{\best}[1]{\cellcolor{bestbg}\textbf{#1}}
\newcommand{\second}[1]{\cellcolor{secondbg}\underline{#1}}

\title{Rethinking the Evaluation and Optimization of LLM-Based Social Simulation\thanks{Project page: \url{https://yupei-wang.github.io/rethinking-social-simulation/}}}

\author{Pei Wang \quad Xu Chen \quad Ji-Rong Wen\\
Gaoling School of Artificial Intelligence, Renmin University of China\\
\texttt{wang\_pei@ruc.edu.cn} \quad \texttt{xu.chen@ruc.edu.cn} \quad \texttt{jrwen@ruc.edu.cn}
}

\iclrfinalcopy
\begin{document}
\maketitle

\begin{abstract}
LLM-based social simulation is a promising complement to traditional
social science methods such as surveys and behavioral experiments. A
core question in this area is how to evaluate the fidelity of
LLM-simulated human behavior and, in turn, how to optimize LLMs toward
it. Prevailing practice evaluates by accuracy, checking whether the
model selects the single response observed from a human, and
accordingly trains the LLM to reproduce this one hard label. However,
human behavior is inherently subjective: the same person in the same
situation may reasonably act in different ways, so an observed
response is only one draw from an underlying response distribution,
which renders accuracy-based evaluation unreliable and hard-label
training misleading. To address these problems, in this paper, we first introduce the
\emph{subjectivity coefficient}, an entropy-based quantity that
distinguishes objective tasks such as coding from subjective tasks such
as social simulation, and then use it to systematically analyze how
accuracy-based evaluation and hard-label training fail as subjectivity
grows. Based on the subjectivity
coefficient, we further propose
\emph{Subjectivity-Adaptive soft-Label Training} (SALT): it pools
observed outputs from semantically nearby inputs into soft
distributional labels, with an aggregation radius adapted to the
estimated subjectivity of each input; in the near-objective limit the
neighborhood shrinks, so SALT naturally falls back to standard
single-label training. Moreover, since
existing datasets record only single observed responses and thus cannot
support distributional evaluation, we construct \textsc{SubjSim}, a
benchmark of 19,300 contexts covering 193 annotators and 100 subjective
questions. Since real-world data typically provide only a single
observation per input, our experiments train models from single observed
outputs while evaluating them against the full response distributions,
thereby verifying the feasibility of our method in realistic settings.
Extensive results on \textsc{SubjSim} demonstrate the advantages of our
method.
\end{abstract}


\section{Introduction}
\label{sec:intro}

LLM-based social simulation holds great potential for social science
research, offering a low-cost, controllable, and scalable way to study
problems such as opinion dynamics, political polarization, market
behavior, and policy interventions
\citep{lu2025beyond,wang2025user,gao2024agentscope,hua2023war,
park2023generative,sasahara2021social,bail2018exposure}. The core question in this area is how to evaluate the fidelity of
LLM-simulated behavior
\citep{dillion2023can,dominguez2024questioning,hu2026simbenchbenchmarkingabilitylarge,
wang2025sociobenchmodelinghumanbehavior,samuel2025personagymevaluatingpersonaagents},
and then, under the established criterion, how to train the model
toward it.
Most existing work approaches this question by directly following the
practice of objective domains such as coding and math: collect one
response for each person on each question, fine-tune the model to
reproduce that response as the correct label, and check with
accuracy-style metrics whether the model picks the same answer.

For objective tasks this recipe is perfectly sound: there is a
well-defined correct answer, so a single label fully specifies the
target, reproducing the label is exactly the desired behavior, and
accuracy faithfully measures success. However, social simulation breaks this
premise.
Human behavior is stochastic rather than deterministic
\citep{fleeson2001toward,camerer1997progress,slovic1995construction,
zaller1992nature,ratcliff1978theory,mcfadden1972conditional,
luce1959individual}: the same person may act differently across
comparable occasions, and people with similar profiles may choose
different actions. An observed response is therefore only one draw from
an underlying response distribution. A single label no longer specifies
the target. During training, reproducing it collapses the model onto
one action; at test time, whether the model hits that answer depends
partly on chance rather than on how faithful the simulation is.

To study the above problems systematically, we first define the
\emph{subjectivity coefficient}, an entropy-based measure of the
inherent randomness of human responses in a given decision context;
under this
definition, objective tasks (e.g., coding) have a near-zero
subjectivity coefficient, while subjective tasks (e.g., social
simulation) have a much higher one. Based
on the subjectivity coefficient, we then formally show that the common
recipe of accuracy evaluation plus single-label training breaks down as
subjectivity grows. For evaluation, the recorded answer is itself only one random draw
from the response distribution, so whether the model matches it
reveals little about fidelity; for training, fitting the single label
pulls the model away from the true response distribution. Finally, we propose
\emph{Subjectivity-Adaptive soft-Label Training} (SALT). For each
context, SALT collects the answers observed at semantically similar
contexts that share the same options and turns them into a soft label,
and the neighborhood size automatically adapts to how subjective the
context is estimated to be. When a context is nearly objective, the
neighborhood shrinks and SALT falls back to standard single-label
training, so hard-label supervision is recovered as the
low-subjectivity special case.

Evaluating distributional fidelity requires reference human response
distributions, but existing social-simulation datasets typically record only
a single response for each context. We therefore construct \textsc{SubjSim}, a
benchmark in which 193 annotators answer 100 subjective survey
questions, yielding 19,300 annotator-question pairs. Each pair carries an elicited
response-propensity distribution obtained through probability-ball
allocation. During training, models never see these distributions and
receive only one derived hard action per pair; at test time, the
full distributions are used to measure more accurately how faithfully
the model reproduces human behavior. With Qwen3-8B as the backbone, SALT reduces aggregate KL
divergence over standard supervised fine-tuning (SFT) by 77.6\%, JSD by 45.9\%, TVD by
31.2\%, and MMD by 52.7\%, with the largest gains in mid- and
high-subjectivity settings, where one observed response reveals the
least about the full distribution.

In summary, this paper makes contributions at three levels.
(i)~At the problem level, we identify why the prevailing accuracy-based
evaluation and training practice is unreasonable for social simulation,
and rigorously analyze this failure by introducing the subjectivity
coefficient. (ii)~At the method level, we propose SALT, which
aggregates observations from semantically similar contexts into soft
distributional labels with a subjectivity-adaptive radius.
(iii)~At the data level, we construct \textsc{SubjSim}, a benchmark of
19,300 annotator-question pairs with human-annotated response
distributions, and extensive experiments on it validate the
effectiveness of SALT.

\section{Preliminaries}
\label{sec:preliminaries}

\subsection{Formulation of Social Simulation}
\label{subsec:formulation}
We describe each individual in the target population by a persona
vector $u$, which encodes demographic features, personality traits,
prior beliefs, or other personal attributes. A decision context is
the tuple $x := (u, s)$, where $s$ is a situational description; we
write $\cX$ for the space of all decision contexts. At
context $x$, the individual chooses from a candidate action set
$\mathcal{A}(x) = \{a^{(1)}, \ldots, a^{(K)}\}$, where each action is a
natural-language description and $K = |\mathcal{A}(x)|$ may vary across
contexts. For example, $u$ may describe a 26-year-old graduate
student living in a big city, $s$ may ask which factor matters most
when choosing a job, and the corresponding action set is
$\cA(x) = \{\text{``salary''}, \text{``stability''},
\text{``personal interest''}, \text{``work--life balance''}\}$. In
another example, $u$ may describe a retired teacher living in a small
town, $s$ may ask how the individual mainly gets news, and
$\cA(x) = \{\text{``social media''}, \text{``television''},
\text{``news apps''}\}$.
Social simulation aims to use an LLM to mimic how these individuals
respond and behave.

\subsection{The Ideal Optimization Target}
\label{subsec:ideal}
Defining the ideal optimization target requires first understanding
the nature of human behavior. Many social science theories have shown that human behavior is not
deterministic but contains an inherent random component. Random
utility theory \citep{luce1959individual,mcfadden1972conditional},
stochastic evidence-accumulation models~\citep{ratcliff1978theory}, and
preference construction research~\citep{slovic1995construction} all
model a person's choice as a draw from a probability distribution
rather than as a fixed answer. This randomness is easy to see in daily
life: the same person may pick a different dish from the same menu on
different days, answer the same survey question differently when asked
twice~\citep{zaller1992nature}, cooperate in one round of an economic
game and defect in the next~\citep{camerer1997progress}, or be
talkative at one gathering and quiet at
another~\citep{fleeson2001toward}. This stochasticity is part of the
target itself rather than label noise. Ideally, the model should
therefore capture the uncertainty of human behavior: for each context,
it should reproduce not only which action is most likely, but the full
probability mass assigned to all plausible actions. Formally, the target is the latent
response-propensity distribution
$\phiTrue(\cdot \mid x) \in \Delta(\mathcal{A}(x))$, namely the action
distribution that would be observed if the same context could be
measured repeatedly under comparable conditions, where
$\Delta(\mathcal{A}(x)) = \{(q_1, \ldots, q_K) : q_k \geq 0,
\sum_k q_k = 1\}$ is the probability simplex over the available
actions and we write $p_k := \phiTrue(k \mid x)$ for brevity. On the
model side, an LLM with parameters $\theta$ induces an action-level
distribution $\phiModel(\cdot \mid x) \in \Delta(\mathcal{A}(x))$,
obtained by normalizing the generation probabilities of the $K$
candidate actions, with $\phiModel(k \mid x)$ written analogously. The
ideal goal of social simulation is to learn a model whose induced
action distribution is close to $\phiTrue$ across contexts:
\begin{equation}
\label{eq:objective}
\hat{\theta} = \argmin_\theta \mathbb{E}_x\!\left[D\!\left(\phiTrue(\cdot \mid x),\; \phiModel(\cdot \mid x)\right)\right],
\end{equation}
where $D$ measures mismatch on the action simplex and can be
instantiated as various distance functions, e.g., total variation
$\dtv(P, Q) := \frac{1}{2}\sum_{k=1}^K |P_k - Q_k|$, cross entropy
$D_{\mathrm{CE}}(P, Q) := -\sum_{k} P_k \log Q_k$, or KL divergence
$\dkl(P \Vert Q)$.

\subsection{The Current Accuracy-Based Practice}
\label{subsec:practice}
In practice, $\phiTrue$ is not observable: each context is recorded
only once, yielding a dataset $\mathcal{D} = \{(x^i, a(x^i))\}$ in
which $a(x) \in \mathcal{A}(x)$ is the single observed action, either
one stochastic draw from $\phiTrue(\cdot \mid x)$ or a single-action
proxy derived from it. Let $j(x)$ denote the index of $a(x)$ and
$\bm{\delta}_{a(x)}$ the one-hot distribution that puts all mass on
$a(x)$. Prevailing practice trains and evaluates against this single
observation. On the training side, the model is trained to maximize
the probability it assigns to the observed action:
\begin{equation}
  \label{eq:sft-objective}
  \tilde{\theta}
  = \arg\max_\theta \mathbb{E}_x\!\left[\log \phiModel(j(x) \mid x)\right]
  = \argmin_\theta \mathbb{E}_x\!\left[D_{\mathrm{CE}}\!\left(\bm{\delta}_{a(x)},\; \phiModel(\cdot \mid x)\right)\right],
\end{equation}
which is exactly objective~\eqref{eq:objective} with the unobservable
target $\phiTrue(\cdot \mid x)$ replaced by the one-hot vector
$\bm{\delta}_{a(x)}$ and $D$ taken as the cross entropy. On the
evaluation side, accuracy checks whether the model's top-probability
action coincides with the same single observation:
\begin{equation}
  \label{eq:accuracy}
  \mathrm{Acc}(\theta)
  = \mathbb{E}_x\!\left[\mathbf{1}\!\left[\textstyle\arg\max_k \phiModel(k \mid x) = j(x)\right]\right],
\end{equation}
replacing the divergence in \eqref{eq:objective} by a 0--1 comparison
between the model's mode and the single draw. A natural question
arises: how much do \eqref{eq:sft-objective} and \eqref{eq:accuracy}
actually capture of the ideal objective~\eqref{eq:objective}? In the
next section, we answer this question by rigorously analyzing the gap
between them.


\section{Misalignment Analysis of the Accuracy-Based Practice}
\label{sec:analysis}

\subsection{Evaluation Metric Analysis}
\label{subsec:eval-failure}
Ideally, an evaluation metric should be consistent with true model
quality: a model that scores better under the metric should also be
closer to the true response distribution under the ideal
objective~\eqref{eq:objective}. We now check whether
accuracy~\eqref{eq:accuracy} has this property.

Consider ranking two
models $\theta_1$ and $\theta_2$ by accuracy at a context $x$ with
observed index $j = j(x)$, and let $k_1$ and $k_2$ denote the two
models' top actions. There are three cases.

(i)~$k_1 = k_2$: the two models receive the same score. However, the
true error $\dtv(\phi_{\theta_i}, \phiTrue)$ depends on all $K$
probability values rather than the top action alone, and the hidden
gap can be nearly maximal. Suppose $\theta_1$ matches the target
exactly, $\phi_{\theta_1} = \phiTrue$, while $\theta_2$ puts all its
mass on the shared top action $k_1$, which is then the mode of
$\phiTrue$ with $p_{k_1} = p_{\max}$. Their true errors are
\begin{equation}
  \label{eq:case1-gap}
  \dtv(\phi_{\theta_1}, \phiTrue) = 0,
  \qquad
  \dtv(\phi_{\theta_2}, \phiTrue)
  = \frac{1}{2}\Bigl[(1 - p_{\max})
  + \sum\nolimits_{k \neq k_1} p_k\Bigr]
  = 1 - p_{\max},
\end{equation}
which differ by $1 - p_{\max}$ and reach $1 - 1/K$ under uniform
behavior, yet accuracy scores the two models identically.

(ii)~$k_1 \neq k_2$ and neither equals $j$: both models score zero,
and accuracy again cannot tell them apart. As in case~(i), consider a
pair of models where one is perfect and the other is a point-mass
model. The observed action $j$ is one random draw from $\phiTrue$ and
may well miss the most probable action, so even the perfect model
$\phi_{\theta_1} = \phiTrue$, whose
top action $k_1$ is the most probable one, can fall into this case,
while $\theta_2$ puts all its mass on another unobserved action,
$\phi_{\theta_2} = \bm{\delta}_{k_2}$. Their true errors are
\begin{equation}
  \label{eq:case2-gap}
  \dtv(\phi_{\theta_1}, \phiTrue) = 0,
  \qquad
  \dtv(\phi_{\theta_2}, \phiTrue)
  = \frac{1}{2}\Bigl[(1 - p_{k_2})
  + \sum\nolimits_{k \neq k_2} p_k\Bigr]
  = 1 - p_{k_2},
\end{equation}
so a perfect model and a model with error $1 - p_{k_2}$ receive the
same zero score.

(iii)~$k_1 \neq k_2$ and exactly one equals $j$, say $k_1 = j$:
accuracy ranks $\theta_1$ higher, and this is the only case where it
expresses a preference. However, the preference can be exactly
backward. Suppose $\theta_2$ is the perfect model,
$\phi_{\theta_2} = \phiTrue$, with its top action $k_2$ the most
probable one, while $\theta_1$ puts all its mass on the draw
$j \neq k_2$. Their true errors are
\begin{equation}
  \label{eq:case3-tie}
  \dtv(\phi_{\theta_1}, \phiTrue)
  = \frac{1}{2}\Bigl[(1 - p_j)
  + \sum\nolimits_{k \neq j} p_k\Bigr]
  = 1 - p_j
  \;>\; 0
  = \dtv(\phi_{\theta_2}, \phiTrue),
\end{equation}
so whenever the single draw misses the most probable action, accuracy
prefers a strictly worse model over the perfect one, and the reversed
gap $1 - p_j$ grows as behavior becomes more diffuse.

In summary, the above counterexamples show that accuracy is not a
reliable measure of behavioral fidelity: models with the same score
can differ substantially in true quality, and in the
worst case accuracy even prefers a strictly worse model over the
perfect one. A better accuracy score therefore does not imply a
model closer to the true response distribution.

\subsection{Training Objective Analysis}
\label{subsec:opt-failure}

We next examine the training objective~\eqref{eq:sft-objective}. Its
empirical signal at each training context is the one-hot target
$\bm{\delta}_{a(x)}$: the objective is monotonically increasing in the
probability assigned to the observed action, so its optimum drives
$\phiModel(j \mid x) \to 1$ on the training context, with no gradient
signal that rewards distributing mass across multiple actions.

\begin{proposition}[Single-observation training concentrates on one-hot labels]
  \label{prop:all-opt}
  Consider a finite training set in which each context $x^i$ is observed
  once with hard label $j^i$. If the closure of the model-induced
  action distributions contains the probability simplex independently at
  each training context (an idealized full-capacity condition), then any empirical-risk minimizer $\tilde{\theta}$ of
  objective~\eqref{eq:sft-objective} in this closure assigns
  $\phi_{\tilde{\theta}}(k \mid x^i) = \mathbf{1}[k = j^i]$ for every
  training context~$x^i$. In other words, the closer a model is trained toward
  the minimum loss, the closer its predicted distribution comes to
  putting probability one on the single observed answer at each
  training context.
\end{proposition}

\noindent
The proof is given in \Cref{app:proof-all-opt}.
Under this point-mass solution, the total-variation error at a
training context is:
\begin{equation}
  \label{eq:tv-error}
  \dtv\!\left(\phi_{\tilde{\theta}}, \phiTrue\right)
  = \frac{1}{2}\sum_{k} \bigl|\mathbf{1}[k=j] - p_k\bigr|
  = 1 - p_j,
\end{equation}
which equals the total probability mass that $\phiTrue$ assigns to
actions other than the one observed. Note that the observed answer $j$
is just one random draw: it tends to be a likely option but is not
always the most likely one, so the error $1 - p_j$ is at least
$1 - p_{\max}$. Even in the most favorable case where the draw hits
the most likely option, diffuse behavior forces every probability,
including the largest, to be small, so the error grows with
subjectivity and reaches $(K-1)/K$ at the uniform distribution.

The analysis above indicates, qualitatively, that the error of the
single-label optimum is ultimately governed by how diffuse the
response distribution is. To describe this relation quantitatively, we
introduce a scalar that summarizes the dispersion of the entire
distribution.

\begin{definition}[Subjectivity coefficient]
  \label{def:kappa}
  The \emph{subjectivity coefficient} of a decision context~$x$ is the
  negentropy of the target distribution:
  \begin{equation}
    \label{eq:kappa}
    \kappa(x)
    \;\coloneqq\;
    \sum_{k=1}^{K} p_k \log p_k
    \;\in\; [-\log K,\; 0],
  \end{equation}
  where $p_k = \phiTrue(k \mid x)$.
\end{definition}

\noindent
The coefficient equals zero when $\phiTrue$ is a point mass (fully
deterministic behavior) and $-\log K$ when it is uniform (maximum
ambiguity).
Throughout, we mainly work with the normalized subjectivity coefficient
\begin{equation}
  \label{eq:subjectivity-score}
  s(x) \coloneqq -\frac{\kappa(x)}{\log K}
  = \frac{H(\phiTrue(\cdot\mid x))}{\log K}
  \in [0,1],
\end{equation}
where $H$ is the Shannon entropy, so larger $s(x)$ means more subjective behavior.
Crucially, $s$ is a property of the decision context~$x$,
not of any model: it characterizes how inherently subjective a given
scenario is.
Under this view, objective tasks such as coding and math occupy the
near-zero-$s$ regime, subjective simulation tasks span the full
range, and different task types differ only in their degree of
subjectivity along this common axis.
Entropy is the standard measure of how uncertain a distribution is,
and it takes all $K$ probabilities into account rather than a single
one. We now show that the training error grows with $s$ through an
explicit lower bound.

\begin{proposition}[Subjectivity lower-bounds the error of
  single-label training]
  \label{prop:kappa-unified}
  Let $p_{\max} = \max_k p_k$, $s = s(x)$, and $K \geq 3$.
  Any point-mass model incurs training error
  \begin{equation}
    \dtv(\phi_\theta, \phiTrue)
    \;\geq\; 1 - p_{\max}
    \;\geq\; \frac{s \log K - \log 2}{\log(K-1)},
    \label{eq:kappa-structural}
  \end{equation}
  where the first inequality follows from \Cref{eq:tv-error} since
  $p_j \leq p_{\max}$, and the second follows from Fano's inequality
  (Appendix~\ref{app:proof-kappa-unified}). Hence the more subjective
  the context, i.e., the larger $s$, the larger the unavoidable error
  of any point-mass model; the lower bound is approximately $s$ for
  large $K$, so the unavoidable error is at least roughly the
  normalized subjectivity of the context. At the uniform distribution,
  where $s = 1$, the exact error $1 - p_{\max}$ equals $1 - 1/K$.
\end{proposition}

The essential cause behind the failures on both the evaluation and the
training side is the same: one observed answer per context carries too
little information about a diffuse response distribution. Ideally, if each context were annotated many
times, the empirical answer frequencies would recover the true
distribution and both problems would disappear. This is unrealistic,
however: reliable model ranking alone would require
$\Omega(K^{2s})$ repeated observations per context
(Appendix~\ref{app:proof-sample-complexity}), while real data provide
exactly one. A classical alternative is to directly merge the observed
answers of similar contexts, as in local smoothing
\citep{nadaraya1964estimating,watson1964smooth}. Merging, however,
involves a trade-off: pooling more neighbors supplies more
distributional information, but neighbors are only similar rather than
identical, so pooling also mixes in answers from different
distributions and biases the target. In the next section, we analyze
this trade-off theoretically and derive an algorithm with a
context-adaptive merging radius.

\section{Subjectivity-Adaptive Soft-Label Training (SALT)}
\label{sec:method}

SALT replaces each one-hot target with a soft label aggregated from
semantically similar contexts, and chooses the neighborhood size to
balance the trade-off above.
Concretely, given $\mathcal{D}=\{(x^i,a(x^i))\}_{i=1}^{n}$, SALT
outputs a soft label $\hat{\phi}(\cdot\mid x)$ for each context. Because
actions are comparable only within a shared candidate set, we first
partition contexts into action-space groups:
\begin{equation}
  \label{eq:action-partition}
  \cG_g \;=\; \bigl\{x \in \cX : \cA(x) = \cA_g\bigr\},
  \qquad g = 1, \ldots, G,
\end{equation}
where $\cA(x)$ is the candidate action set of context~$x$ defined in
\Cref{subsec:formulation} and $\cA_g$ is the shared candidate set of
group~$g$; in survey data, for instance, all contexts answering the
same question with the same options form one group.
Within each group, contexts are embedded with a pretrained encoder and
compared by $\ell_2$ distance~$d$. The neighborhood of $x\in\cG_g$ is
\begin{equation}
  \label{eq:neighborhood}
  \cN(x) \;=\; \bigl\{x' \in \cG_g : d(x', x) \leq r(x)\bigr\},
\end{equation}
where the radius $r(x)$ is adaptive; how to choose it is the key
design question, addressed below. Since all neighbors share $\cA_g$, their hard
actions define a local empirical distribution:
\begin{equation}
  \label{eq:soft-label}
  \hat{\phi}(k \mid x)
  \;\coloneqq\;
  \frac{1}{|\cN(x)|}
  \sum_{x' \in \cN(x)} \mathbf{1}[a(x') = a^{(k)}],
  \qquad k = 1, \ldots, K,
\end{equation}
and the model is trained to match this distributional target:
\begin{equation}
  \label{eq:agg-loss}
  \mathcal{L}^{\mathrm{agg}}_\theta
  \;=\; \sum_{x \in \cD}
    \dkl\!\left(\hat{\phi}(\cdot \mid x) \;\big\|\;
    \phiModel(\cdot \mid x)\right),
\end{equation}
where $\phiModel(\cdot \mid x) \in \Delta(\cA(x))$ is obtained by
normalizing generation probabilities across candidates. Note that this
differs from prior grouping-based methods
\citep{huang-etal-2026-distribution, cao2025specializing}, which
partition samples into disjoint groups (e.g., by demographic
attributes) and let all samples in a group share one target
distribution, whereas SALT centers a neighborhood at each context, so
every context receives its own soft label.

The key remaining question is how to determine the radius $r(x)$. To
answer it, we bound the error of aggregation within one action-space
group $\cG_g$, writing $d_{\cX}$ for the
intrinsic dimension of its context space, $n_g=|\cG_g|$ for the number of
contexts in the group, and $K$ for the number of actions. The
dimension enters through the volume of a neighborhood: a ball of
radius $r$ in a $d_{\cX}$-dimensional space holds a fraction
$\asymp r^{d_{\cX}}$ of the contexts, so
$|\cN(x)| \asymp n_g \cdot r^{d_{\cX}}$. The bound rests on a smoothness assumption that formalizes
the intuition behind SALT, namely that semantically nearby contexts
induce similar response distributions; it is stated as an
$L$-Lipschitz condition in the embedding distance $d$
(\Cref{assump:lipschitz} in Appendix~\ref{app:aggregation-lemmas}).
The statistical part of the bound depends on how many actions carry
substantial probability, captured by the effective action count.

\begin{definition}[Effective number of actions]
  \label{def:keff}
  For a context $x$, the \emph{effective number of actions} is
  \begin{equation}
    \label{eq:keff-def}
    \Keff^{*}(x) \;\coloneqq\;
    \Bigl(\sum_{k=1}^{K} \sqrt{\phiTrue(k \mid x)}\Bigr)^{2}
    \;\in\; [1, K],
  \end{equation}
  which attains its lower endpoint when $\phiTrue(\cdot \mid x)$ is a
  point mass and its upper endpoint when $\phiTrue(\cdot \mid x)$ is
  uniform. It further satisfies $\Keff^{*}(x) \geq K^{s(x)}$, so it
  increases with the subjectivity of the context; this bound and the
  remaining properties used below are established in
  \Cref{remark:keff-properties}.
\end{definition}

Combining the decomposition with statistical and bias bounds yields the
main tradeoff.

\begin{theorem}[Aggregation--Estimation Tradeoff]
  \label{thm:main}
  Suppose $\phiTrue$ is $L$-Lipschitz in the embedding distance $d$
  (Appendix~\ref{app:aggregation-lemmas}). With
  $|\cN(x)| \asymp n_g \cdot r^{d_{\cX}}$, where $\asymp$
  ($\lesssim$) denotes equality (inequality) up to constant factors,
  let
  $\varepsilon_{\mathrm{opt}}
  \coloneqq \E[\dtv(\phiModel(\cdot\mid x),\hat{\phi}(\cdot\mid x))]$.
  The expected error satisfies
  \begin{equation}
    \label{eq:total-bound}
    \E\!\left[
      \dtv\!\left(\phiModel(\cdot \mid x),\;
      \phiTrue(\cdot \mid x)\right)
    \right]
    \;\lesssim\;
    \underbrace{\varepsilon_{\mathrm{opt}}}_{\text{opt.\ error}}
    \;+\;
    \underbrace{L \cdot r}_{\text{bias}}
    \;+\;
    \underbrace{\sqrt{\frac{\Keff^{*}(x)}{n_g \cdot r^{d_{\cX}}}}
    }_{\text{stat.\ error}}.
  \end{equation}
  The bias increases in $r$ while the statistical error decreases in
  $r$ (since $|\cN(x)|$ grows with $r^{d_\cX}$).
  When the aggregation loss is optimized so that
  $\varepsilon_{\mathrm{opt}}$ is negligible, balancing the remaining two
  terms gives the bias--variance optimized radius and the corresponding
  optimized upper bound:
  \begin{equation}
    \label{eq:optimal-r}
    r^* \;\asymp\; \left(\frac{\Keff^{*}(x)}{n_g L^2}
    \right)^{\!1/(d_{\cX}+2)},
    \qquad
    \E[\dtv]
    \;\lesssim\;
    \varepsilon_{\mathrm{opt}} \;+\;
    \left(\frac{L^{d_{\cX}} \cdot \Keff^{*}(x)}{n_g}
    \right)^{\!1/(d_{\cX}+2)}.
  \end{equation}
  In the worst case ($\Keff^{*}(x) = K$), the bias--statistical term becomes
  $(L^{d_{\cX}} K / n_g)^{1/(d_{\cX}+2)}$.
\end{theorem}

\noindent
The proof is given in \Cref{app:proof-main}. The result has two
implications. First, with negligible optimization error, oracle
aggregation improves as $n_g$ grows, whereas a local point-mass fit still
incurs error at least $1-p_{\max}$ on the context
(\Cref{eq:kappa-structural}). Second, $r^*$ increases with
$\Keff^{*}(x)$: the more subjective a context, the larger its
neighborhood should be.

\paragraph{Practical implementation of the oracle radius.}
\Cref{thm:main} gives an oracle radius
$r^* \asymp (\Keff^{*}(x)/(n_g L^2))^{1/(d_\cX+2)}$.
Since $\Keff^{*}(x)$ is defined through $\phiTrue(\cdot \mid x)$, it is
unobserved in training. In practice we evaluate the \emph{same}
expression at the model's current distribution, which gives the
model-based estimate
$\hat{K}_{\mathrm{eff}}(x) \coloneqq
\bigl(\sum_k \sqrt{\phiModel(k \mid x)}\bigr)^2$.
This is reasonable for two reasons. The aggregation loss
(\Cref{eq:agg-loss}) is itself a distribution-matching objective, so it
drives $\phiModel(\cdot \mid x)$ toward $\phiTrue(\cdot \mid x)$, and
the expression is continuous in the distribution, so the estimate
returns to the true count as the optimization error vanishes. Exact
recovery is moreover not required, because $r(x)$ depends on the count
only through the power $1/(d_{\cX}+2)$, so a multiplicative error in
the count is damped into a much smaller relative change in the radius.
\Cref{remark:plugin-keff} makes both precise.
The Lipschitz constant $L$ and other constants are absorbed into a
tunable hyperparameter $C$, giving
\begin{equation}
  \label{eq:adaptive-radius-practical}
  r(x) \;=\; C \cdot \left(\frac{\hat{K}_{\text{eff}}(x)}{n_g}\right)^{1/(d_{\mathcal{X}}+2)}.
\end{equation}
The estimate is refreshed periodically during training. This substitution preserves the
oracle bound's monotone dependence on action dispersion, and we evaluate
it through fixed-$\Keff$, fixed-neighborhood, and radius-scale ablations.
The theory thus does not merely justify aggregation; it specifies when
to aggregate more: contexts with diffuse predicted behavior require
larger neighborhoods to reduce statistical error, while concentrated
contexts should remain close to their observed hard label to avoid
unnecessary smoothing. In the zero-subjectivity limit,
$\hat{K}_{\mathrm{eff}}(x) \to 1$ and the radius contracts toward its minimum, so
the soft label concentrates on the observed action and SALT approaches
standard hard-label training. Hard-label supervision is therefore
recovered as the low-subjectivity special case of SALT. More details and
the complete algorithm can be found in Appendix~\ref{app:ca-details}.

\section{Experiments}
\label{sec:experiments}

\subsection{SubjSim Benchmark}
\label{sec:data}
Evaluating distributional fidelity requires the true response
distributions as ground truth, which existing single-response datasets
cannot provide. We therefore construct \textsc{SubjSim}, where 193
annotators first answer 30 demographic questions that define their
persona vectors and then annotate 100 subjective survey questions,
yielding 19,300 persona-question contexts. For each context, the annotator allocates
probability balls across the candidate options \citep{van1993eli,
delavande2010eliciting}, producing a response-propensity distribution
that is used only for evaluation. All training methods, including
SALT's aggregation, see just one answer per context: the option
receiving the largest share of that annotator's balls, with ties
resolved in favor of the lowest-indexed option. The questions are
organized into eight topic domains: economy, politics, technology,
social issues, culture, health, environment, and education. Their
subjectivity coefficients
cover the full range from near-deterministic to near-uniform, enabling
evaluation across low, mid, and high subjectivity regimes (see
\Cref{fig:dataset_stats}; construction details in
Appendix~\ref{sec:dataset}).

\subsection{Experimental Setup}
\label{subsec:exp-setup}
We train on the samples of about 85\% of the respondents and test on
those of the remaining ones, which gives 16,500 training pairs and
2,800 test pairs.
All experiments use Qwen3-8B \citep{yang2025qwen3} as the backbone.
We compare SALT with the untrained zero-shot backbone, SFT, DPO, PPO, and DSA
\citep{huang-etal-2026-distribution}. DPO constructs preference pairs and
PPO derives reward signals from the hard labels described in
\Cref{sec:data}, so that all methods observe identical training
information. DSA is the most recent distribution-level
baseline, which fine-tunes the LLM to match the response distributions
of demographic groups and to align distribution shifts across groups.
At test time, models output distributions
over candidate options via generation-probability normalization. We
report KL, JSD, TVD, and
linear-kernel MMD computed on the full test set; see
Appendix~\ref{app:ca-details} for further experimental settings.

\subsection{Main Results}
\label{subsec:exp-main}

The main results are shown in
\Cref{tab:main-results}. The pretrained backbone performs poorly, and
neither DPO nor PPO improves over it; DPO in fact degrades KL
substantially. We attribute this to the construction of the preference
signal from a single hard action: the observed action serves as the
positive and the remaining options as negatives, yet under subjective
behavior these options may themselves carry substantial probability, so
responses the respondent might well choose are suppressed as negatives,
making the resulting signal even noisier than fitting the label
directly. SFT is the strongest baseline, suggesting that directly
fitting the observed labels already captures a substantial part of the
underlying behavioral structure. DSA ranks between the preference-based
methods and SFT, clearly outperforming the former, which confirms the
benefit of distribution-alignment training. However, DSA groups
respondents by only a few discrete background attributes, and all
individuals within a group share a single target distribution; the
supervision is therefore noisy at the individual level, leaving DSA
behind SFT.
SALT controls this noise with its adaptive radius, pooling only
sufficiently similar contexts while still collecting enough answers
for a reliable distribution estimate. It performs best
on all four metrics in every domain, and over the full test set it
reduces KL over SFT by 77.6\%,
JSD by 45.9\%, TVD by 31.2\%, and MMD by 52.7\%. A paired bootstrap over
test contexts confirms that all of these improvements are statistically
significant (Appendix~\ref{app:significance}).

\begin{table*}[!t]
\centering
\caption{Main results on SubjSim. All metrics are divergences (lower is better). Best values are in \textbf{bold} on a \colorbox{bestbg}{blue} background; second-best values are underlined on a \colorbox{secondbg}{sand} background.}
\label{tab:main-results}
\setlength{\tabcolsep}{3.0pt}
\renewcommand{\arraystretch}{0.95}
\scriptsize
\resizebox{\textwidth}{!}{%
\begin{tabular}{lcccccccccccc}
\toprule
& \multicolumn{4}{c}{\textbf{Economy}} & \multicolumn{4}{c}{\textbf{Political}} & \multicolumn{4}{c}{\textbf{Technology}} \\
\cmidrule(lr){2-5}\cmidrule(lr){6-9}\cmidrule(lr){10-13}
\textbf{Method} & KL$\downarrow$ & JSD$\downarrow$ & TVD$\downarrow$ & MMD$\downarrow$ & KL$\downarrow$ & JSD$\downarrow$ & TVD$\downarrow$ & MMD$\downarrow$ & KL$\downarrow$ & JSD$\downarrow$ & TVD$\downarrow$ & MMD$\downarrow$ \\
\midrule
Pretrained & 4.6285 & 0.2851 & 0.5735 & 0.6346 & 4.1453 & 0.3044 & 0.6206 & 0.6586 & 4.2015 & 0.2547 & 0.5335 & 0.5502 \\
SFT & \second{1.3543} & \second{0.1542} & \second{0.3951} & \second{0.3283} & \second{1.2865} & \second{0.1568} & \second{0.4145} & \second{0.3167} & \second{1.3409} & \second{0.1537} & \second{0.3956} & \second{0.3215} \\
DPO & 6.9962 & 0.2752 & 0.5573 & 0.6049 & 7.2926 & 0.2834 & 0.5816 & 0.5861 & 6.9174 & 0.2580 & 0.5305 & 0.5450 \\
PPO & 4.0161 & 0.2672 & 0.5495 & 0.5863 & 4.4698 & 0.2726 & 0.5724 & 0.5684 & 4.2959 & 0.2502 & 0.5251 & 0.5310 \\
DSA & 3.2032 & 0.2332 & 0.4920 & 0.4337 & 3.4679 & 0.2442 & 0.5161 & 0.4436 & 3.4744 & 0.2337 & 0.4888 & 0.4275 \\
\addlinespace[1pt]
SALT (Ours) & \best{0.3129} & \best{0.0892} & \best{0.2839} & \best{0.1640} & \best{0.2535} & \best{0.0724} & \best{0.2473} & \best{0.1247} & \best{0.2817} & \best{0.0815} & \best{0.2736} & \best{0.1466} \\
\midrule
& \multicolumn{4}{c}{\textbf{Social}} & \multicolumn{4}{c}{\textbf{Culture}} & \multicolumn{4}{c}{\textbf{Health}} \\
\cmidrule(lr){2-5}\cmidrule(lr){6-9}\cmidrule(lr){10-13}
\textbf{Method} & KL$\downarrow$ & JSD$\downarrow$ & TVD$\downarrow$ & MMD$\downarrow$ & KL$\downarrow$ & JSD$\downarrow$ & TVD$\downarrow$ & MMD$\downarrow$ & KL$\downarrow$ & JSD$\downarrow$ & TVD$\downarrow$ & MMD$\downarrow$ \\
\midrule
Pretrained & 4.2286 & 0.2720 & 0.5562 & 0.6079 & 3.0456 & 0.2445 & 0.5175 & 0.5646 & 4.7412 & 0.2822 & 0.5737 & 0.6467 \\
SFT & \second{1.3259} & \second{0.1588} & \second{0.3977} & \second{0.3354} & \second{0.8578} & \second{0.1209} & \second{0.3459} & \second{0.2739} & \second{1.5169} & \second{0.1566} & \second{0.3886} & \second{0.3286} \\
DPO & 6.8092 & 0.2608 & 0.5289 & 0.5629 & 5.6298 & 0.2385 & 0.5050 & 0.5398 & 6.5058 & 0.2501 & 0.5152 & 0.5366 \\
PPO & 4.2388 & 0.2524 & 0.5210 & 0.5471 & 3.5294 & 0.2278 & 0.4909 & 0.5180 & 4.4975 & 0.2409 & 0.5059 & 0.5215 \\
DSA & 3.2049 & 0.2334 & 0.4963 & 0.4537 & 4.1157 & 0.2623 & 0.5259 & 0.5762 & 3.8155 & 0.2433 & 0.4889 & 0.4561 \\
\addlinespace[1pt]
SALT (Ours) & \best{0.3168} & \best{0.0905} & \best{0.2863} & \best{0.1656} & \best{0.2593} & \best{0.0746} & \best{0.2652} & \best{0.1548} & \best{0.2995} & \best{0.0859} & \best{0.2793} & \best{0.1601} \\
\midrule
& \multicolumn{4}{c}{\textbf{Environment}} & \multicolumn{4}{c}{\textbf{Education}} & \multicolumn{4}{c}{\textbf{ALL}} \\
\cmidrule(lr){2-5}\cmidrule(lr){6-9}\cmidrule(lr){10-13}
\textbf{Method} & KL$\downarrow$ & JSD$\downarrow$ & TVD$\downarrow$ & MMD$\downarrow$ & KL$\downarrow$ & JSD$\downarrow$ & TVD$\downarrow$ & MMD$\downarrow$ & KL$\downarrow$ & JSD$\downarrow$ & TVD$\downarrow$ & MMD$\downarrow$ \\
\midrule
Pretrained & 5.3499 & 0.3091 & 0.6199 & 0.6506 & 6.1097 & 0.3662 & 0.6842 & 0.8295 & 4.2801 & 0.2851 & 0.5745 & 0.6350 \\
SFT & \second{0.8800} & \second{0.1373} & \second{0.3837} & \second{0.2685} & \second{2.0897} & 0.2324 & 0.5246 & 0.4818 & \second{1.2871} & \second{0.1524} & \second{0.3953} & \second{0.3194} \\
DPO & 7.2403 & 0.3099 & 0.6163 & 0.6573 & 8.4950 & 0.3276 & 0.6291 & 0.7074 & 6.8819 & 0.2690 & 0.5497 & 0.5774 \\
PPO & 3.5411 & 0.2914 & 0.6035 & 0.6206 & 5.1310 & 0.3337 & 0.6447 & 0.7245 & 4.1393 & 0.2588 & 0.5403 & 0.5578 \\
DSA & 3.0895 & 0.2116 & 0.4613 & 0.3496 & 2.7515 & \second{0.2260} & \second{0.4870} & \second{0.4152} & 3.3904 & 0.2343 & 0.4937 & 0.4418 \\
\addlinespace[1pt]
SALT (Ours) & \best{0.2983} & \best{0.0870} & \best{0.2705} & \best{0.1407} & \best{0.3866} & \best{0.1095} & \best{0.3229} & \best{0.1947} & \best{0.2880} & \best{0.0825} & \best{0.2720} & \best{0.1510} \\
\bottomrule
\end{tabular}%
 }
\end{table*}

\begin{figure*}[!t]
\centering
\includegraphics[width=\textwidth]{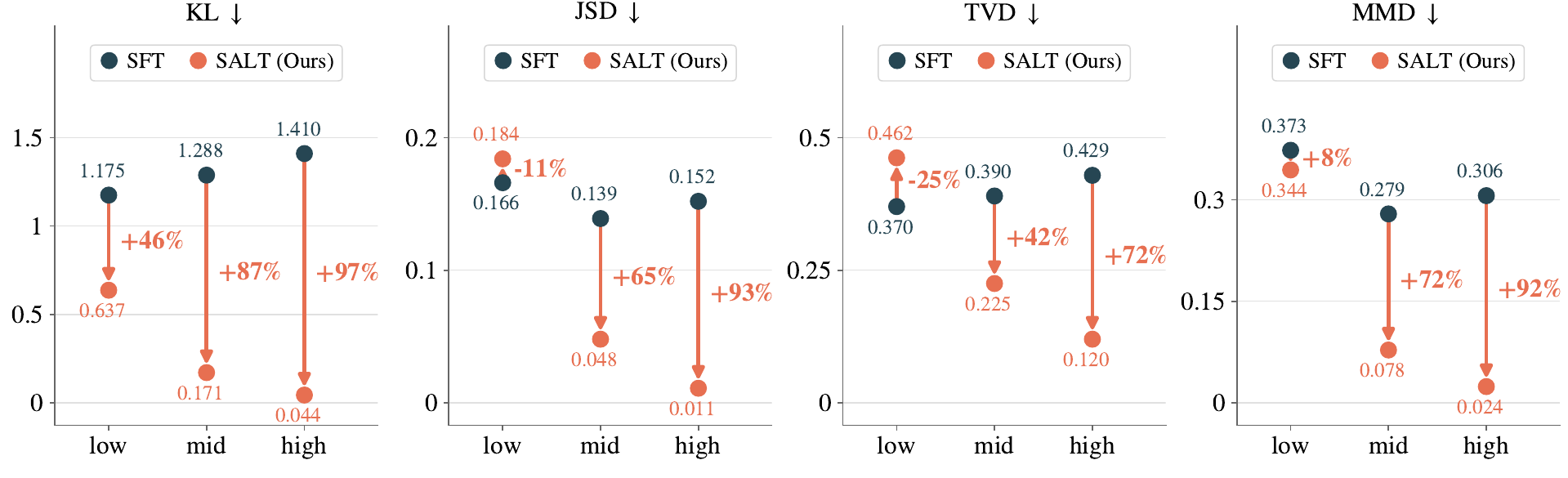}
\caption{Subjectivity-stratified comparison between SFT and SALT on
\textsc{SubjSim}, over low-, mid-, and high-subjectivity questions.
All metrics are divergences (lower is better), and percentages give
SALT's relative change over SFT.}
\label{fig:subjectivity-levels}
\end{figure*}

\subsection{Performance Comparison across Subjectivity Levels}
\label{subsec:exp-subjectivity}
A core design of SALT is that the aggregation is calibrated by the
estimated subjectivity of each context, and we now examine
experimentally whether this design indeed leads to better
performance. We compute
the normalized subjectivity coefficient $s(x)$
in~\eqref{eq:subjectivity-score} directly from the annotated
response-propensity distributions. We
then partition the
questions into low, mid, and high strata at the $1/3$ and $2/3$
quantiles of $s(x)$ and compare SFT and SALT within each stratum
(\Cref{fig:subjectivity-levels}). The results match the analysis in
\Cref{sec:analysis}. SFT is competitive in near-deterministic contexts
but degrades steadily as targets become diffuse, whereas SALT's
advantage grows with subjectivity; in the mid and high strata SALT
outperforms SFT on all metrics, and in the high stratum it reduces KL
by 96.9\%. In the lowest-subjectivity regime the comparison is mixed,
with SALT improving KL and MMD but falling behind SFT on JSD and TVD. This is what we would expect,
because such questions have a clear majority answer, so there is little
for aggregation to add and pooling neighbors can only blur a target
that is already sharp.

\subsection{Ablation Studies}
\label{sec:ablation}
The core design of SALT is the context-adaptive merging radius, so we
ask whether such adaptivity is really necessary. We build three
baselines that differ from SALT only in how neighbors are selected:
Top-$N$ gives every context the same number of nearest neighbors; fixed
$K_{\text{eff}}$ replaces $\hat{K}_{\text{eff}}(x)$ in the radius
formula by one shared constant, so every context gets the same radius;
Global-Freq drops similarity altogether and treats everyone who
answered the same question as a neighbor.

From the results shown in \Cref{fig:ablation}, we draw three
conclusions. First, every variant improves substantially over SFT,
which shows that aggregation is already useful on its own because it
replaces a single observed answer with a distributional target.
Second, how neighbors are selected matters: Top-$N$ degrades as $N$
grows and dissimilar contexts enter the neighborhood, fixed
$K_{\text{eff}}$ stops improving and stays worse than SALT on every
metric, and the adaptive $\hat{K}_{\text{eff}}(x)$ is best on all four.
Third, Global-Freq is clearly worse than SALT, so similarity is
necessary: pooling everyone who answered the same question mixes
dissimilar respondents and biases the soft label.

\begin{figure}[t]
\centering
\includegraphics[width=0.82\linewidth]{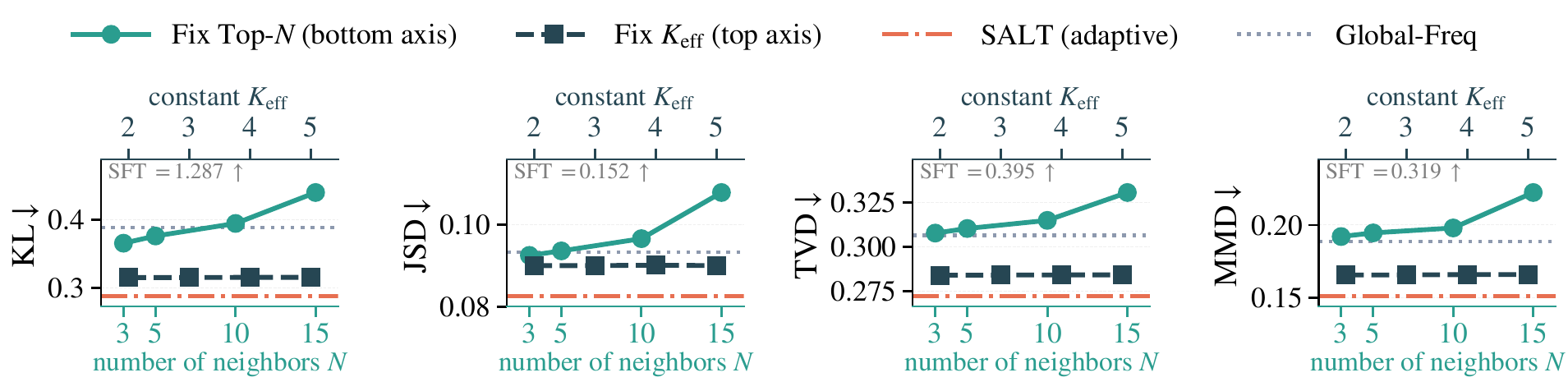}
\caption{Ablation over neighbor selection. Top-$N$ varies $N$ (teal,
bottom axis); fixed $K_{\text{eff}}$ substitutes a constant for
$\hat{K}_{\text{eff}}(x)$ (navy, top axis). Both share the $y$ axis
and reference lines; SFT lies far above the plotted range and is
reported as text.}
\label{fig:ablation}
\end{figure}

\begin{figure}[t]
\centering
\includegraphics[width=0.82\linewidth]{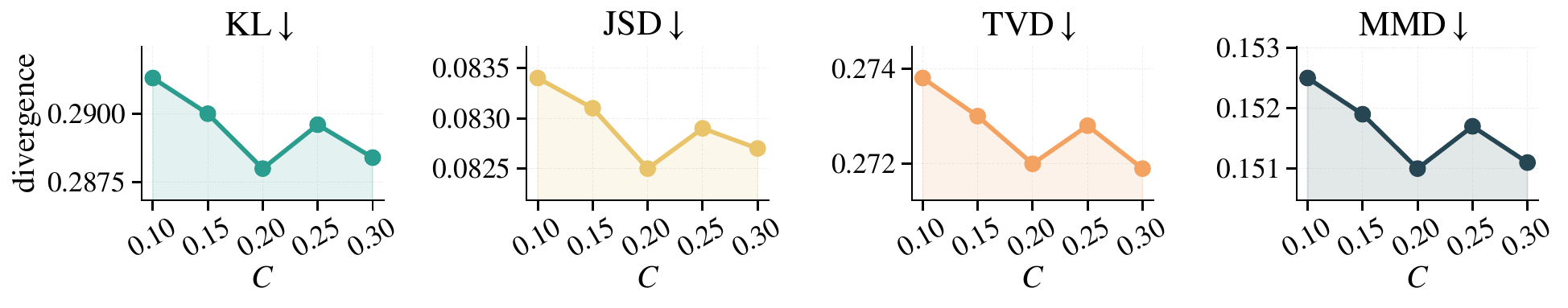}
\caption{Sensitivity of SALT to the radius scale $C$ (divergences,
lower is better).}
\label{fig:sensitivity-C}
\end{figure}

\FloatBarrier

\subsection{Parameter Analysis}
\label{subsec:param-analysis}
The main hyperparameter of SALT is the radius scale $C$
in~\eqref{eq:adaptive-radius-practical}. The theory in
\Cref{sec:method} fixes only how the radius should scale with
$\hat{K}_{\mathrm{eff}}(x)$ and $n_g$; the remaining constants,
including the Lipschitz constant $L$, are absorbed into $C$, which is
therefore chosen empirically. \Cref{fig:sensitivity-C} reports
performance for
$C \in \{0.10,0.15,0.20,0.25,0.30\}$. All metrics improve
monotonically as $C$ increases from 0.10 to 0.20: a larger radius
admits more genuinely similar neighbors, so the aggregated soft labels
become more reliable. Beyond 0.20 the improvement stops and fluctuates
slightly, as a wider radius includes less similar contexts and mild
oversmoothing offsets the gains; the trade-off is best around
$C = 0.20$, which we use as the default. Importantly, the variation
across the whole range is small---about 0.003 in KL and at most 0.002
in the other metrics---so SALT does not rely on a finely tuned radius
constant.

\section{Conclusion}
\label{sec:conclusion}
In this paper, we studied how to evaluate and train LLMs for social
simulation when human behavior is subjective. We introduced the
subjectivity coefficient, which places objective and subjective tasks
on a common axis and explains why accuracy-based evaluation and
hard-label training become unreliable as subjectivity grows. We then
proposed SALT, which turns each observed response into a soft
distributional label aggregated over a subjectivity-adaptive
neighborhood, and built \textsc{SubjSim}, whose annotators provide
full response distributions for 19,300 persona-question contexts,
making distributional evaluation possible. SALT improves
distributional alignment over all baselines, with the largest gains in
high-subjectivity regimes. Future work includes extending this
framework to sequential, multi-agent, and open-ended settings.

\section*{Ethics Statement}
\label{sec:ethics}
\textsc{SubjSim} involves human annotation. All 193 annotators
participated voluntarily; before starting the annotation task, they
were informed of the purpose of the study, the type of data collected,
and their right to withdraw at any time. Annotation compensation was
calculated so that the resulting hourly rate is guaranteed to exceed
the highest applicable local hourly wage standard. To protect privacy,
no directly identifying information (such as names or contact details)
was collected, all responses
were recorded under anonymized identifiers, and the released data contain only
demographic attribute values and annotated response distributions with
no link back to individuals. The survey questions themselves were
screened for cultural suitability for the annotator population during
dataset construction, and questions flagged as unsuitable were removed
(Appendix~\ref{sec:dataset}). The
representativeness limits of our annotator pool are discussed in
Appendix~\ref{subsec:annotator}.

\section*{AI Use Statement}
\label{sec:ai-use}
We used large language models only as writing and figure assistants in
the preparation of this paper: ChatGPT was used to check grammar and
correct typos in the manuscript, and Codex was used to assist in
producing the figures. No LLM was used to generate research ideas,
analyses, results, or claims. Separately, and as part of the
research methodology itself rather than paper preparation, DeepSeek-chat
was used to translate and pre-screen survey questions during the
construction of \textsc{SubjSim}, as documented in
Appendix~\ref{sec:dataset}.

\section*{Reproducibility Statement}
\label{sec:reproducibility}
We have taken several measures to make our results reproducible. All
theoretical results are stated with their assumptions in
\Cref{sec:analysis,sec:method}, and complete proofs are given in
\Cref{app:proof-all-opt,app:proof-kappa-unified,app:proof-sample-complexity,app:aggregation-lemmas,app:proof-bounds,app:proof-main}.
The full SALT procedure is specified in
\Cref{alg:context-agg}, with implementation details, the
theory-to-implementation mapping (\Cref{tab:theory-implementation}),
and all training hyperparameters (\Cref{tab:hyperparams}) in
\Cref{app:ca-details}. The construction of \textsc{SubjSim}, including source
surveys, screening and translation pipeline, the exact screening
prompt, annotation protocol, and annotator demographics, is
documented in \Cref{sec:dataset}. Evaluation is deterministic, and the backbone model (Qwen3-8B), context encoder
(Qwen3-embedding-8b), and training stack (LLaMA-Factory, DeepSpeed
ZeRO-2 on 8$\times$H20) are publicly available. We will release the
\textsc{SubjSim} dataset, the annotation platform specification, and
the training and evaluation code upon publication.

\bibliography{iclr2027_conference}
\bibliographystyle{iclr2027_conference}

\newpage
\appendix
\begin{center}
{\Large\bf Appendix}
\end{center}
\startcontents[appendix]
\printcontents[appendix]{}{1}{\setcounter{tocdepth}{2}}
\newpage

\section{Limitations}
\label{app:limitations}
\textsc{SubjSim} uses elicited response propensities rather than direct
repeated-choice frequencies. The annotator-level split tests unseen
personas within known questions rather than transfer to entirely new
decision contexts. SALT also depends on behavioral smoothness in the
representation space, so poor embeddings or discontinuous response
patterns can make aggregation harmful.

\section{Related Work and Positioning}
\label{app:related}

\paragraph{LLM agents for social simulation.}
Recent work uses LLM agents to simulate individuals, groups, and social
systems, ranging from interactive generative agents to opinion dynamics,
polarization, market behavior, and historical or policy simulations
\citep{park2023generative,gao2024agentscope,hua2023war,
sasahara2021social,bail2018exposure}. This line of work establishes that
LLMs can produce plausible social behavior and can be embedded in
multi-agent environments. Our focus is different: we ask what behavioral
target such agents should be trained and evaluated against. Rather than
treating simulation quality as matching a single observed response, we
formulate the target as recovering the full distribution over plausible
actions for a persona-context pair.

\paragraph{Single-action imitation and preference optimization.}
The dominant way to adapt LLM agents is supervised imitation of observed
actions, often evaluated by point-prediction accuracy. Preference
optimization methods such as DPO and PPO provide a stronger alignment
toolkit by learning from chosen/rejected comparisons
\citep{rafailov2023direct,schulman2017proximal}. However, when the data
provide only one realized action per context, both supervised and
preference objectives still receive hard, mode-like supervision. They
therefore learn which action was observed, or which action should be
preferred to alternatives, rather than the entire behavioral frequency
structure. This is the same hard-supervision gap that motivates our
misalignment analysis, though our formal result focuses on the
single-observation SFT objective.

\paragraph{Distributional evaluation and distribution alignment.}
A separate line of work changes the evaluation target: instead of
reporting only accuracy, it compares model outputs to response
distributions using KL, JSD, TVD, Wasserstein distance, or related
survey-response divergences. DSA is closest to our work in target space:
it also aims to match survey response distributions, but it does so
through distribution-shift alignment rather than context-level hard-label
aggregation~\citep{huang-etal-2026-distribution}. Our setting differs in the
supervision assumption. DSA and related distributional alignment methods
rely on distributional supervision that is available or can be constructed
at the population/question level. SALT targets the stricter
single-observation regime: it does not train on the persona-specific
probability-ball distribution, but constructs distributional supervision
by pooling hard observations from semantically similar contexts, following
the repeated-measurement intuition in the main analysis.

\paragraph{Human label variation.}
NLP research on annotation has similarly argued that disagreement among
annotators carries signal rather than noise, advocating soft labels,
annotator-aware models, and evaluation against label distributions instead
of a single gold label \citep{aroyo2015truth,plank2022human,uma2021learning}.
These works concern classification-style NLP annotations and model
variation at the label or annotator level. Our setting differs in both
object and ground truth: the estimation target is a persona-conditional
behavioral distribution for social simulation, and
\textsc{SubjSim} measures it directly through probability-ball
allocation rather than approximating it from a handful of discrete
annotations.

\paragraph{Context aggregation and nonparametric estimation.}
SALT is related to nonparametric smoothing and kernel regression, where
local neighborhoods are used to estimate a conditional response function
\citep{nadaraya1964estimating,watson1964smooth,silverman2018density}. The
key difference is the object being smoothed and the constraint imposed by
social-simulation data. We aggregate over contexts that share an action
space and are close in persona-context embedding space, producing a soft
label over discrete actions from otherwise hard observations. The adaptive
radius is not a generic hyperparameter search: it is derived from the
subjectivity-controlled bias--variance tradeoff, so more diffuse
behavioral distributions receive broader aggregation while nearly
deterministic contexts remain local.

\section{Theoretical Proofs}
\subsection{Proof of \texorpdfstring{\Cref{prop:all-opt}}{Proposition} (Single-Observation SFT
  Concentrates on Empirical One-Hot Labels)}
\label{app:proof-all-opt}

\begin{proof}
  For each training context $x^i$, the empirical SFT contribution is
  $-\log \phiModel(j^i \mid x^i)$ at the action level. Over the
  action-distribution simplex, this term is minimized by maximizing
  $\phiModel(j^i \mid x^i)$, whose largest possible value is~$1$.
  Because the statement is made in the closure of the model-induced
  action distributions, this boundary point is included and can be
  selected independently for every training context, yielding
  $\phiModel(k \mid x^i)=\mathbf{1}[k=j^i]$ for all~$i$.
  Since $-\log p$ is uniquely minimized at $p=1$, every
  empirical-risk minimizer must saturate the per-context term at every
  training context, so any minimizer $\tilde{\theta}$ in this closure
  satisfies $\phi_{\tilde{\theta}}(k \mid x^i)=\mathbf{1}[k=j^i]$ for
  all~$i$.
  For finite softmax parameterizations, the same boundary point may be
  approached only in the limit: any sequence with empirical risk
  approaching the infimum must have
  $\phiModel(j^i\mid x^i)\to 1$ and hence
  $\phiModel(k\mid x^i)\to 0$ for every $k\neq j^i$.
  This is an empirical-risk statement for finite single-observation
  data; with repeated observations from the same context, the population
  cross-entropy optimum would instead match the conditional label
  distribution.
\end{proof}

\subsection{Proof of the Fano Step in \texorpdfstring{\Cref{prop:kappa-unified}}{the Unified Proposition}}
\label{app:proof-kappa-unified}

\begin{proof}
  Write $e := 1 - p_{\max}$ and
  $H := H(\phiTrue(\cdot \mid x)) = -\kappa$. Fano's inequality for
  the error of the optimal single guess (the mode) states
  $H \leq h(e) + e \log(K-1)$, where $h(\cdot)$ is the binary entropy
  function. Since $h(e) \leq \log 2$, we obtain
  $H \leq \log 2 + e \log(K-1)$, and rearranging gives
  \begin{equation}
    1 - p_{\max} = e \;\geq\; \frac{H - \log 2}{\log(K-1)}
    = \frac{-\kappa - \log 2}{\log(K-1)}
    = \frac{s \log K - \log 2}{\log(K-1)}
  \end{equation}
  for $K \geq 3$, where the last step uses
  $s = -\kappa/\log K$ from \Cref{eq:subjectivity-score}. Combining
  with $\dtv(\phi_\theta, \phiTrue) \geq 1 - p_{\max}$ from
  \Cref{eq:tv-error} yields \Cref{eq:kappa-structural}. The bound is
  vacuous when $s \log K < \log 2$, i.e., for weakly subjective
  contexts, which is consistent with the analysis: the failures of the
  accuracy-based pipeline concentrate in the high-subjectivity regime.
  For $K = 2$, Fano's inequality reduces to $H \leq h(e)$, giving
  $e \geq h^{-1}(H)$ instead; the conclusion that the error grows with
  subjectivity is unchanged.
\end{proof}

\subsection{Proof of \texorpdfstring{\Cref{prop:sample-kappa}}{Proposition} (Local Lower Bound under
  Near-Uniform Behavior)}
\label{app:proof-sample-complexity}

\begin{proposition}[Local lower bound under near-uniform behavior]
  \label{prop:sample-kappa}
  Fix a $K$-action context and consider true distributions in a
  near-uniform neighborhood of radius $O(1/K)$, so that
  $e^{-2\kappa} = \Theta(K^2)$. There exist two alternatives $P,Q$ in
  this neighborhood and two fixed point-mass models whose $\dtv$-error
  ranking is reversed under $P$ and $Q$, while
  $\dkl(P \| Q)=O(1/K^2)$. Consequently, for any fixed
  $\eta>0$, any procedure that identifies the correct ranking from
  $n$ i.i.d.\ observations with error probability at most
  $\frac{1}{2}-\eta$ in this local problem requires
  $n=\Omega(K^2)=\Omega(e^{-2\kappa})$ samples.
\end{proposition}

\begin{proof}
  We give an explicit two-point construction.  Let $u$ be the uniform
  distribution on $K$ actions and choose
  $\varepsilon = cK^{-3/2}$ for a sufficiently small constant $c>0$.
  Define
  \begin{align}
    P_1 &= \frac{1}{K}+\varepsilon, &
    P_2 &= \frac{1}{K}-\varepsilon, &
    P_k &= \frac{1}{K}\quad (k\geq 3), \\
    Q_1 &= \frac{1}{K}-\varepsilon, &
    Q_2 &= \frac{1}{K}+\varepsilon, &
    Q_k &= \frac{1}{K}\quad (k\geq 3).
  \end{align}
  For large enough $K$, both distributions are valid and lie in an
  $O(1/K)$ neighborhood of $u$. Their entropy satisfies
  $H(P)=H(Q)=\log K-O(K\varepsilon^2)=\log K-O(1/K^2)$, so
  $\kappa=-\log K+O(1/K^2)$ and therefore
  $e^{-2\kappa}=\Theta(K^2)$.

  Consider two point-mass models, $\theta_1$ assigning all mass to
  action~$1$ and $\theta_2$ assigning all mass to action~$2$.  For any
  distribution $R$, the total variation error of the point mass at
  action~$j$ is
  \begin{equation}
    \dtv(\delta_j,R)=1-R_j.
  \end{equation}
  Hence, under $P$,
  $\dtv(\delta_1,P)=1-P_1 < 1-P_2=\dtv(\delta_2,P)$, so
  $\theta_1$ is better than $\theta_2$. Under $Q$ the inequality is
  reversed. Any procedure that identifies the correct ranking in this
  local problem therefore distinguishes whether the samples came from
  $P$ or $Q$.

  It remains to bound the statistical distance between the two
  hypotheses. Let $a=1/K$ and $t=\varepsilon/a=c/\sqrt{K}$. The
  per-sample KL divergence is
  \begin{align}
    \dkl(P\|Q)
    &= (a+\varepsilon)\log\frac{a+\varepsilon}{a-\varepsilon}
       +(a-\varepsilon)\log\frac{a-\varepsilon}{a+\varepsilon} \\
    &= 2\varepsilon \log\frac{1+t}{1-t}.
  \end{align}
  Since $t<1/2$ for large enough $K$,
  $\log((1+t)/(1-t)) \leq C t$ for a universal constant $C$, and thus
  \begin{equation}
    \dkl(P\|Q) \leq 2C\frac{\varepsilon^2}{a}
    = O(K\varepsilon^2)=O(1/K^2).
  \end{equation}

  Let $P^n$ and $Q^n$ denote the $n$-sample product distributions. By
  tensorization and Pinsker's inequality,
  \begin{equation}
    \dtv(P^n,Q^n)
    \leq \sqrt{\frac{\dkl(P^n\|Q^n)}{2}}
    = \sqrt{\frac{n\,\dkl(P\|Q)}{2}}
    \leq C'\sqrt{\frac{n}{K^2}} .
  \end{equation}
  Le~Cam's lemma gives minimax error at least
  $\frac{1}{2}(1-\dtv(P^n,Q^n))$. Therefore, if a ranking procedure
  achieves error probability at most $\frac{1}{2}-\eta$ for fixed
  $\eta>0$, then $\dtv(P^n,Q^n)\geq 2\eta$, which requires
  $n\geq c_\eta K^2$. Since $e^{-2\kappa}=\Theta(K^2)$ in this
  construction, the required sample size is
  $\Omega(e^{-2\kappa})$.
\end{proof}

\subsection{Supporting Lemmas for Context Aggregation}
\label{app:aggregation-lemmas}

We first record the properties of the effective action count that
motivate \Cref{def:keff}.

\begin{remark}[Properties of the effective action count]
  \label{remark:keff-properties}
  \Cref{def:keff} attaches the effective action count to a context
  through $\phiTrue(\cdot \mid x)$. The statements below apply it to
  other distributions as well, so throughout the appendix we write
  \begin{equation}
    \label{eq:keff-functional}
    \Keff(\phi) \;\coloneqq\;
    \Bigl(\sum_{k=1}^{K} \sqrt{\phi(k)}\Bigr)^{2},
    \qquad \phi \in \Delta(\cA),
  \end{equation}
  for the same expression evaluated at an arbitrary distribution, so
  that $\Keff^{*}(x) = \Keff(\phiTrue(\cdot \mid x))$ and the
  model-based estimate of \Cref{eq:adaptive-radius-practical} is
  $\hat{K}_{\mathrm{eff}}(x) = \Keff(\phiModel(\cdot \mid x))$.

  The range $1 \leq \Keff(\phi) \leq K$, with the two endpoints
  attained at a point mass and at the uniform distribution, is verified
  inside the proof in \Cref{app:proof-bounds}. Beyond the range,
  $\Keff$ is exactly the R\'enyi perplexity of order $1/2$, that is
  $\Keff(\phi) = e^{H_{1/2}(\phi)}$ where $H_{1/2}$ is the R\'enyi
  entropy of that order. Since R\'enyi entropies are non-increasing in
  their order, $H_{1/2}(\phi) \geq H(\phi)$ for the Shannon entropy
  $H$, and therefore
  \begin{equation}
    \label{eq:keff-lower-bound}
    \Keff(\phiTrue(\cdot \mid x)) \;\geq\;
    e^{H(\phiTrue(\cdot \mid x))} \;=\; K^{s(x)},
  \end{equation}
  with $s(x)$ the normalized subjectivity of
  \Cref{eq:subjectivity-score}. The effective action count is thus
  lower-bounded by $K$ raised to the normalized subjectivity, so the
  two grow together as behavior becomes more diffuse. This is why
  $\Keff$, rather than any other dispersion summary, is a natural
  effective-action measure for the statistical term of
  \Cref{lem:bounds}: the quantity that controls the statistical error
  is governed by the same notion of subjectivity that drives the rest
  of the analysis.

  Finally, $\Keff$ is continuous on the simplex: for any
  $\phi, \psi \in \Delta(\cA)$, using
  $|\sqrt{a} - \sqrt{b}| \leq \sqrt{|a - b|}$ termwise and then
  Cauchy--Schwarz,
  \begin{equation}
    \label{eq:keff-continuity}
    \left|\sqrt{\Keff(\phi)} - \sqrt{\Keff(\psi)}\right|
    \;\leq\; \sum_{k=1}^{K} \sqrt{|\phi(k) - \psi(k)|}
    \;\leq\; \sqrt{K \sum_{k=1}^{K} |\phi(k) - \psi(k)|}
    \;=\; \sqrt{2K \cdot \dtv(\phi, \psi)}.
  \end{equation}
  Two distributions that are close in total variation therefore have
  comparable effective action counts.
\end{remark}

The analysis rests on the following smoothness assumption.

\begin{assumption}[$L$-Lipschitz behavioral distribution]
  \label{assump:lipschitz}
  There exists $L > 0$ such that for all $x, x' \in \cX$:
  $\dtv(\phiTrue(\cdot \mid x),\, \phiTrue(\cdot \mid x'))
  \leq L \cdot d(x, x')$.
\end{assumption}

\noindent
This assumption should be read as behavioral smoothness in the chosen
representation space. It requires that, within a fixed action-space
group, nearby persona-context embeddings induce similar response
distributions; the ablations over fixed neighborhoods, global
frequency labels, and radius scale test whether this approximation is
useful in \textsc{SubjSim}.

\begin{lemma}[Three-term decomposition]
  \label{lem:decomp}
  For any context $x$, let
  $\bar{\phi}(\cdot \mid x) \coloneqq
  \frac{1}{|\cN(x)|}\sum_{x' \in \cN(x)} \phiTrue(\cdot \mid x')$
  denote the population mean of $\phiTrue$ within the neighborhood
  of~$x$. Then:
  \begin{align}
    \label{eq:decomp}
    &\dtv\!\left(\phiModel(\cdot \mid x),\;
      \phiTrue(\cdot \mid x)\right) \notag\\
    &\quad\leq\;
    \underbrace{\dtv(\phiModel(\cdot \mid x),\;
      \hat{\phi}(\cdot \mid x))
    }_{\varepsilon_{\mathrm{opt}}}
    +\;
    \underbrace{\dtv(\hat{\phi}(\cdot \mid x),\;
      \bar{\phi}(\cdot \mid x))
    }_{\varepsilon_{\mathrm{stat}}}
    +\;
    \underbrace{\dtv(\bar{\phi}(\cdot \mid x),\;
      \phiTrue(\cdot \mid x))
    }_{\varepsilon_{\mathrm{bias}}}.
  \end{align}
\end{lemma}

\begin{lemma}[Statistical and bias bounds]
  \label{lem:bounds}
  Under \Cref{assump:lipschitz}, with neighborhood radius~$r$:
  \begin{equation}
    \label{eq:stat-bound}
    \E\!\left[\varepsilon_{\mathrm{stat}}\right]
    \;\leq\; \frac{1}{2}\sqrt{\frac{\Keff}{|\cN(x)|}},
    \qquad
    \varepsilon_{\mathrm{bias}} \;\leq\; L \cdot r,
  \end{equation}
  where $\Keff = \Keff(\bar{\phi}(\cdot \mid x))$ is the effective
  action count of the neighborhood-averaged population distribution
  $\bar{\phi}$ of \Cref{lem:decomp}. Under \Cref{assump:lipschitz},
  $\bar{\phi}(\cdot \mid x)$ lies within $L r$ in total variation of
  $\phiTrue(\cdot \mid x)$, so by \eqref{eq:keff-continuity} it differs
  from $\Keff^{*}(x)$ of \Cref{def:keff} only through the same
  smoothness that already controls the bias term; the main statements
  are written with $\Keff^{*}(x)$ for readability.
  As a worst case one may substitute $\Keff = K$.
\end{lemma}

\begin{corollary}[Aggregation eliminates the structural error of
  standard training]
  \label{cor:improvement}
  In the single-observation regime, standard SFT fits one hard label per
  context and therefore incurs the point-mass error
  $\dtv(\phi_\theta^{\mathrm{SFT}}, \phiTrue) \geq 1 - p_{\max}$ at that
  context.
  Context-aggregation achieves error
  $\varepsilon_{\mathrm{opt}}+
  (L^{d_\cX} \Keff^{*}(x) / n_g)^{1/(d_\cX+2)}$ up to constants; when
  optimization error is
  negligible, this upper bound tends to zero as $n_g \to \infty$.
\end{corollary}

\begin{remark}[Estimating the effective action count from the model]
  \label{remark:plugin-keff}
  The theorem is stated for the oracle effective action count of the
  local population distribution. The practical rule in
  \Cref{eq:adaptive-radius-practical} replaces it with
  $\hat{K}_{\mathrm{eff}}(x)$ computed from the current model output. If
  $\hat{K}_{\mathrm{eff}}(x) \in
  [\Keff^{*}(x)/c,\; c \cdot \Keff^{*}(x)]$ for some constant
  $c \geq 1$, then, since $r(x)$ scales as $\Keff^{1/(d_\cX+2)}$, the
  resulting radius is within a factor $c^{1/(d_\cX+2)}$ of the oracle
  radius, and the optimized upper bound changes only by constants. The
  exponent makes this tolerance generous in practice: with
  $d_\cX = 28$, even $c = 2$ gives $2^{1/30} \approx 1.02$.
  Without this approximation, the model-based rule should be read
  as an oracle-motivated heuristic rather than a direct consequence of
  \Cref{thm:main}.

  The approximation is what the training objective is designed to
  deliver. Applying the continuity bound
  \eqref{eq:keff-continuity} with $\phi = \phiModel(\cdot \mid x)$ and
  $\psi = \bar{\phi}(\cdot \mid x)$, the gap between the model-based
  count and the oracle count is controlled by exactly the quantity the
  aggregation loss in \Cref{eq:agg-loss} minimizes. In the ideal case
  where the optimization error $\varepsilon_{\mathrm{opt}}$ and the
  statistical error of the soft label both vanish, $\phiModel(\cdot \mid
  x) \to \bar{\phi}(\cdot \mid x)$ and hence
  $\hat{K}_{\mathrm{eff}}(x) \to \Keff(\bar{\phi}(\cdot \mid x))$, so
  the practical radius returns to the oracle radius.
\end{remark}

\subsection{Proof of \texorpdfstring{\Cref{lem:bounds}}{Lemma} (Statistical and Bias Bounds)}
\label{app:proof-bounds}

\begin{proof}
  \textbf{Statistical bound.}
  Write $n_x=|\cN(x)|$ and, for $x'\in\cN(x)$,
  $p_{x',k}=\phiTrue(k\mid x')$.
  Each coordinate $\hat{\phi}_k$ is an average of independent Bernoulli
  variables with possibly different means $p_{x',k}$.
  Let
  $\bar{\phi}_k=n_x^{-1}\sum_{x'\in\cN(x)} p_{x',k}$.
  Then
  \begin{equation}
    \text{Var}(\hat{\phi}_k)
    = \frac{1}{n_x^2}\sum_{x'\in\cN(x)}
      p_{x',k}(1-p_{x',k})
    \leq \frac{\bar{\phi}_k}{n_x}.
  \end{equation}
  By linearity of expectation and Jensen's inequality
  ($\E[|X|] \leq \sqrt{\E[X^2]}$):
  \begin{align}
    \E\!\left[\varepsilon_{\mathrm{stat}}\right]
    &= \frac{1}{2}\sum_{k=1}^{K}
       \E\!\left[|\hat{\phi}_k - \bar{\phi}_k|\right]
    \;\leq\; \frac{1}{2}\sum_{k=1}^{K}
       \sqrt{\text{Var}(\hat{\phi}_k)} \\
    &\leq \frac{1}{2}\sum_{k=1}^{K}
       \sqrt{\frac{\bar{\phi}_k}{n_x}}
    \;=\; \frac{1}{2\sqrt{n_x}}\sum_{k=1}^{K}\sqrt{\bar{\phi}_k}.
  \end{align}
  By \eqref{eq:keff-functional},
  $\sum_k \sqrt{\bar{\phi}_k} = \sqrt{\Keff(\bar{\phi}(\cdot\mid x))}$, so
  $\E[\varepsilon_{\mathrm{stat}}] \leq
  \frac{1}{2}\sqrt{\Keff / n_x}$.
  Note that by Cauchy--Schwarz,
  $\sum_k \sqrt{\bar{\phi}_k} \leq
  \sqrt{K \sum_k \bar{\phi}_k} = \sqrt{K}$
  (since $\sum_k \bar{\phi}_k = 1$), so $\Keff \leq K$ always
  holds.
  When $\bar{\phi}(\cdot\mid x)$ is a point mass,
  $\sum_k \sqrt{\bar{\phi}_k} = 1$ so $\Keff = 1$; when
  $\bar{\phi}(\cdot\mid x)$ is uniform,
  $\sum_k \sqrt{1/K} = \sqrt{K}$ so $\Keff = K$.

  \textbf{Bias bound.}
  For the overlapping neighborhood used by SALT, every
  $x'\in\cN(x)$ satisfies $d(x',x)\leq r$.
  By convexity of $\dtv$ and \Cref{assump:lipschitz}:
  \begin{align}
    \varepsilon_{\mathrm{bias}}
    &= \dtv(\bar{\phi}(\cdot\mid x), \phiTrue(\cdot \mid x)) \\
    &\leq
      \frac{1}{|\cN(x)|}\sum_{x'\in\cN(x)}
      \dtv(\phiTrue(\cdot \mid x'), \phiTrue(\cdot \mid x)) \\
    &\leq
      \frac{1}{|\cN(x)|}\sum_{x'\in\cN(x)} L\cdot d(x',x)
      \;\leq\; Lr.
  \end{align}
\end{proof}

\subsection{Proof of \texorpdfstring{\Cref{thm:main}}{Theorem} (Aggregation--Estimation Tradeoff)}
\label{app:proof-main}

\begin{proof}
  Substituting $|\cN(x)| \asymp n_g r^{d_{\cX}}$ into \Cref{lem:bounds} and
  combining via \Cref{lem:decomp} gives
  \begin{equation}
    \E[\dtv] \;\lesssim\; \varepsilon_{\mathrm{opt}} + Lr +
    \sqrt{\frac{\Keff}{n_g r^{d_{\cX}}}}.
  \end{equation}
  The optimization term is independent of the neighborhood radius in this
  tradeoff. The bias term increases in $r$ and the statistical term
  decreases in $r$.
  To find the optimal $r$, we differentiate with respect to $r$ and
  set the result to zero:
  \begin{equation}
    L \;=\; \frac{d_{\cX}}{2} \cdot
    \frac{1}{r} \cdot \sqrt{\frac{\Keff}{n_g r^{d_{\cX}}}},
  \end{equation}
  which gives $L^2 r^{d_{\cX}+2} \asymp \Keff/n_g$, yielding
  \begin{equation}
    r^* \;\asymp\;
    \left(\frac{\Keff}{n_g L^2}\right)^{1/(d_{\cX}+2)}.
  \end{equation}
  Substituting $r^*$ back: the bias term is
  $Lr^* = L \cdot (\Keff/(n_g L^2))^{1/(d_{\cX}+2)}
  = (L^{d_{\cX}} \Keff / n_g)^{1/(d_{\cX}+2)}$,
  and the statistical term is of the same order, giving the optimized
  upper bound in~\eqref{eq:optimal-r}.
\end{proof}

\section{SALT Implementation Details}
\label{app:ca-details}

On \textsc{SubjSim}, SALT groups contexts by survey question, embeds
each context with Qwen3-embedding-8b, retrieves adaptive-radius
neighbors within the group, and trains on the resulting soft labels.
The remainder of this section specifies each of these steps.

\begin{algorithm}[!t]
  \caption{Context-Aggregation Training with Adaptive Radius}
  \label{alg:context-agg}
  \begin{algorithmic}[1]
    \REQUIRE Dataset $\cD$; pretrained context encoder; radius function
             $\rho$; refresh interval $T$; number of epochs $E$
    \ENSURE Trained model $\phiModel$
    \STATE \textbf{Action-space partitioning:} group all contexts by
           action space: $\cG_g = \{x : \cA(x) = \cA_g\}$
    \STATE Embed all contexts using the pretrained encoder
    \STATE Initialize $\phiModel$ from a pretrained LLM
    \FOR{epoch $= 1, \ldots, E$}
      \FOR{each optimizer step}
        \IF{the step index is a multiple of $T$}
          \FOR{each context $x$}
            \STATE Compute adaptive radius
                   $r(x) \leftarrow \rho(\phiModel, x)$
            \STATE Retrieve neighborhood
                   $\cN(x) \leftarrow \{x' \in \cG_g : d(x', x) \leq r(x)\}$
            \STATE Construct soft label
                   $\hat{\phi}(k \mid x) \leftarrow
                   |\cN(x)|^{-1}\sum_{x' \in \cN(x)}
                   \mathbf{1}[a(x') = a^{(k)}]$
          \ENDFOR
        \ENDIF
        \STATE Update $\theta$ on the current batch by minimizing
               $\dkl(\hat{\phi}(\cdot \mid x) \| \phiModel(\cdot \mid x))$
      \ENDFOR
    \ENDFOR
  \end{algorithmic}
\end{algorithm}

\begin{table*}[!t]
\centering
\caption{Mapping between theoretical concepts, notation, and their SubjSim implementation.}
\label{tab:theory-implementation}
\small
\setlength{\tabcolsep}{4pt}
\renewcommand{\arraystretch}{1.15}
\begin{tabularx}{\textwidth}{@{}
  >{\raggedright\arraybackslash}p{2.5cm}
  >{\centering\arraybackslash}p{2.0cm}
  >{\raggedright\arraybackslash}X
  >{\raggedright\arraybackslash}p{4.0cm}
  @{}}
\toprule
\textbf{Concept} & \textbf{Notation} & \textbf{SubjSim realization} & \textbf{Protocol role} \\
\midrule
Decision context & $x=(u,s)$ & Annotator persona paired with a survey question; each persona has 30 demographic attributes. & Input to training and evaluation. \\
Action space & $\cA(x)$ & Candidate response options for the survey question. & Common support for distributional evaluation within each question. \\
Latent response-propensity target & $\phiTrue(\cdot\mid x)$ & Probability-ball empirical distribution $\hat{\phi}^{\mathrm{ball}}_u(\cdot\mid x)$, used as an elicited distributional proxy rather than a repeated-choice frequency. & Hidden during training; used only as the evaluation target. \\
Hard observation & $a(x)$ & Modal response under $\hat{\phi}^{\mathrm{ball}}_u(\cdot\mid x)$. & Hard-label proxy for SFT, DPO, PPO, and SALT targets. \\
Action-space group & $\cG_g$ & Persona-question contexts from the same survey question. & Restricts aggregation to comparable candidate options. \\
SALT neighborhood & $\cN(x)$ & Nearest contexts within the adaptive embedding radius, using Qwen3-embedding-8b embeddings. & Defines which hard observations SALT pools. \\
Effective action count (oracle) & $\Keff^{*}(x)$ & Effective action count of the true response distribution at the context; never observed. & Appears in the bound of \Cref{thm:main} and in the oracle radius. \\
Effective action count (model-based) & $\hat{K}_{\mathrm{eff}}(x)$ & The same functional applied to the model's own distribution, $\Keff(\phiModel(\cdot\mid x))$. & Sets the adaptive radius without using the hidden target. \\
Global baseline & Global-Freq & Question-level frequency of hard labels. & Ablation that removes persona conditioning. \\
\bottomrule
\end{tabularx}
\end{table*}

\paragraph{Context embeddings.}
We use Qwen3-embedding-8b \citep{zhang2025qwen3embedding} as the context encoder, which produces
4096-dimensional vectors for Chinese text. Embeddings are precomputed
once before training and cached on disk; they are not updated during
training. All contexts within the same action-space group share the same
situational description~$s$. As a result, the embeddings primarily
capture persona similarity within each group rather than situational
variation.

\paragraph{Neighborhood Construction}
Pairwise distances between context embeddings are computed using the
Euclidean ($\ell_2$) distance. The full distance matrix is precomputed
once at training startup using SciPy on CPU and cached for reuse.
We estimate the intrinsic dimension $d_{\cX}$ of the persona embedding
space via PCA with a variance threshold of $0.90$, yielding
$d_{\cX} = 28$.
The adaptive radius $r(x)$ is then computed per context according to
Equation~\ref{eq:adaptive-radius-practical}, with scale factor
$C = 0.2$.

\paragraph{Soft-label refresh schedule.}
Soft labels are refreshed every 30 training steps using the current
model checkpoint. Per-epoch refresh adapts slowly in early training,
whereas per-step refresh is computationally prohibitive and unstable.
The 30-step interval balances label responsiveness with training
efficiency.

\paragraph{Training hyperparameters.}
Table~\ref{tab:hyperparams} reports the main training and
method-specific hyperparameters for each method. All methods are
trained on 8 GPUs; SALT uses per-device batch size 1 with 16
gradient-accumulation steps. Both the DPO and PPO policies are initialized from
the SFT checkpoint. SALT uses a maximum sequence length of 1024 for
evaluation (\texttt{ca\_eval\_max\_length}).

\begin{table}[t]
\centering
\caption{Main training and method-specific hyperparameters for all methods.}
\label{tab:hyperparams}
\small
\setlength{\tabcolsep}{3.5pt}
\begin{tabular}{@{}llcccccc@{}}
\toprule
& & \multicolumn{5}{c}{\textbf{Policy models}} & \textbf{PPO aux.} \\
\cmidrule(lr){3-7}\cmidrule(lr){8-8}
\textbf{Group} & \textbf{Parameter} & \textbf{SALT} & \textbf{SFT} & \textbf{DPO} & \textbf{PPO} & \textbf{DSA} & \textbf{RM} \\
\midrule
\multirow{7}{*}{\textbf{General}}
& Epochs & 4 & 5 & 4 & 4 & 4 & 4 \\
& Global batch size & 128 & 128 & 96 & 128 & 128 & 128 \\
& Learning rate & 5e-6 & 5e-6 & 5e-7 & 1e-6 & 5e-6 & 1e-5 \\
& LR schedule & Cosine & Cosine & Cosine & Cosine & Cosine & Cosine \\
& Warmup ratio & 0.1 & 0.1 & 0.1 & 0.1 & 0.1 & 0.1 \\
& Optimizer & AdamW & AdamW & AdamW & AdamW & AdamW & AdamW \\
& Max sequence length & 1024 & 1024 & 1024 & 1024 & 1024 & 1024 \\
\midrule
\multirow{4}{*}{\textbf{DPO}}
& $\beta$ & --- & --- & 0.07 & --- & --- & --- \\
& Loss type & --- & --- & Sigmoid & --- & --- & --- \\
& Label smoothing & --- & --- & 0.0 & --- & --- & --- \\
& FTX coefficient & --- & --- & 0.0 & --- & --- & --- \\
\midrule
\multirow{5}{*}{\textbf{PPO}}
& $\epsilon_{\text{clip}}$ & --- & --- & --- & 0.2 & --- & --- \\
& Target KL & --- & --- & --- & 6.0 & --- & --- \\
& Initial KL coeff. & --- & --- & --- & 0.05 & --- & --- \\
& Sampling temperature & --- & --- & --- & 0.7 & --- & --- \\
& Top-$p$ & --- & --- & --- & 0.9 & --- & --- \\
\midrule
\multirow{3}{*}{\textbf{SALT}}
& Scale factor $C$ & 0.2 & --- & --- & --- & --- & --- \\
& PCA variance threshold & 0.90 & --- & --- & --- & --- & --- \\
& Softmax temperature & 1.0 & --- & --- & --- & --- & --- \\
\bottomrule
\end{tabular}
\end{table}

\paragraph{Reward Model Construction (PPO).}
The reward model is initialized from the SFT checkpoint. For each
context, the action with the highest probability in the annotator's
empirical distribution, i.e., the hard label defined in
\Cref{sec:data}, is treated as the chosen response. Each remaining
candidate is paired individually as a rejected response, yielding $K-1$
preference pairs per context.

\paragraph{Computational Cost.}
SALT and SFT are trained on eight NVIDIA H20 GPUs using DeepSpeed ZeRO-2.
Context embeddings for all 16,500 training samples are precomputed in
40 seconds. Each training epoch takes approximately 35 minutes for SALT
and 31 minutes for SFT.

\FloatBarrier

\section{Per-Domain Results}
\label{app:per-domain}

\Cref{fig:subjectivity-levels,fig:ablation,fig:sensitivity-C} in the main
text aggregate over the whole test set. This section repeats the same three
analyses within each of the eight topic domains of \textsc{SubjSim}, so that
the subjectivity trend, the neighbor-selection ablation, and the radius-scale
sensitivity can each be checked domain by domain.

\begin{figure}[p]
\centering
\begin{subfigure}{\linewidth}
\centering
\includegraphics[width=\linewidth]{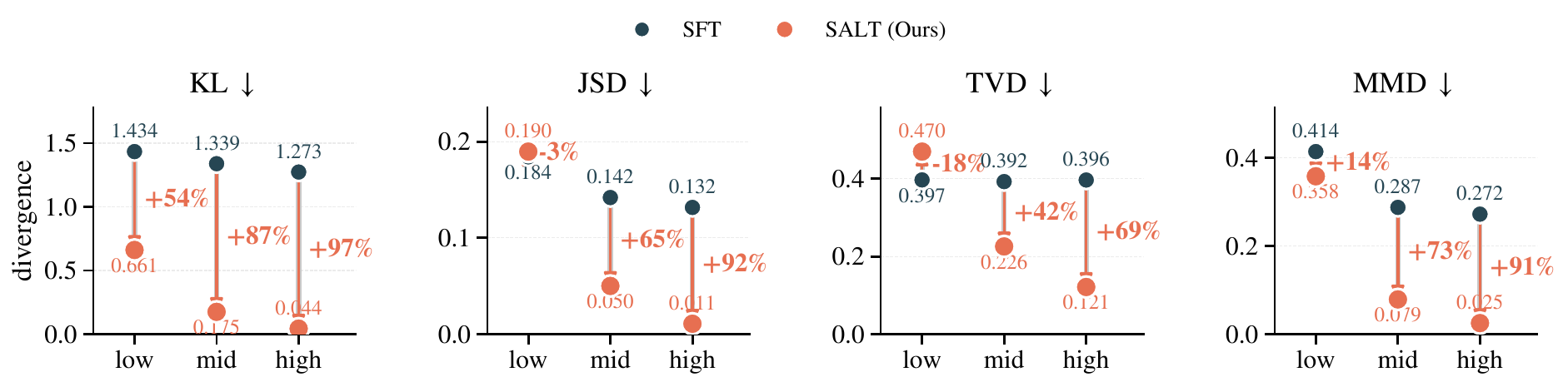}
\caption{Subjectivity-stratified comparison between SFT and SALT}
\end{subfigure}

\vspace{8pt}
\begin{subfigure}{\linewidth}
\centering
\includegraphics[width=\linewidth]{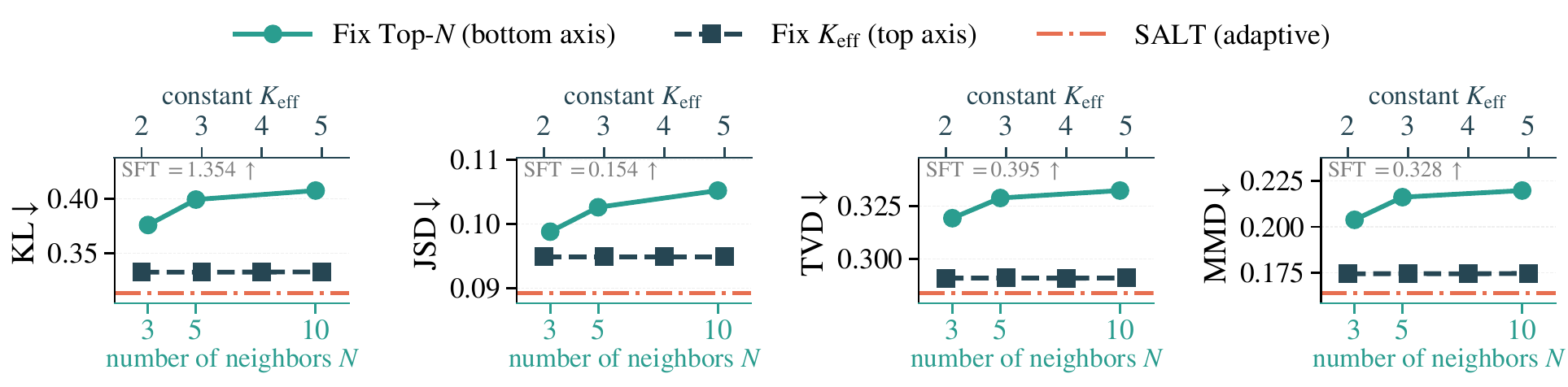}
\caption{Ablation over neighbor selection}
\end{subfigure}

\vspace{8pt}
\begin{subfigure}{\linewidth}
\centering
\includegraphics[width=\linewidth]{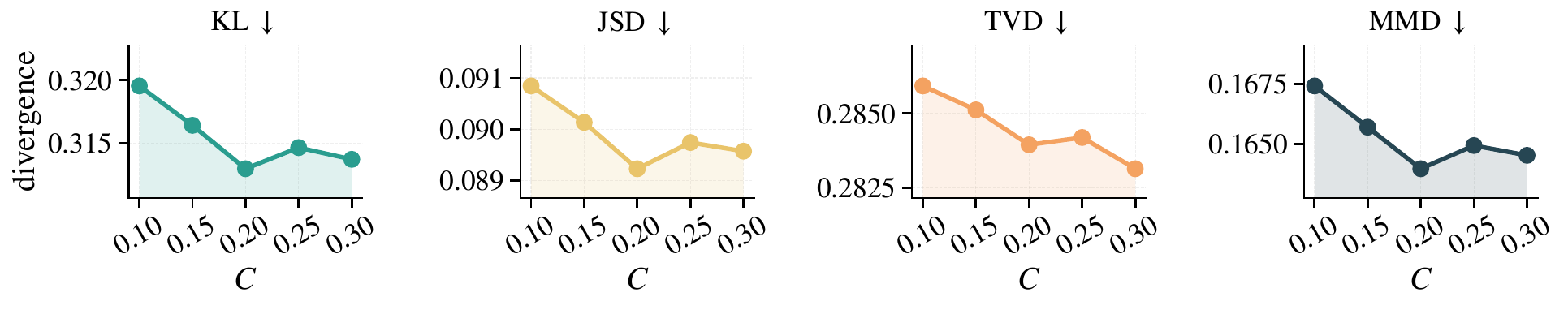}
\caption{Sensitivity to the radius scale $C$}
\end{subfigure}
\caption{Per-domain results on the \textbf{Economy} domain, mirroring
\Cref{fig:subjectivity-levels}, \Cref{fig:ablation}, and
\Cref{fig:sensitivity-C} of the main text. All metrics are divergences
(lower is better).}
\label{fig:domain-economy}
\end{figure}

\begin{figure}[p]
\centering
\begin{subfigure}{\linewidth}
\centering
\includegraphics[width=\linewidth]{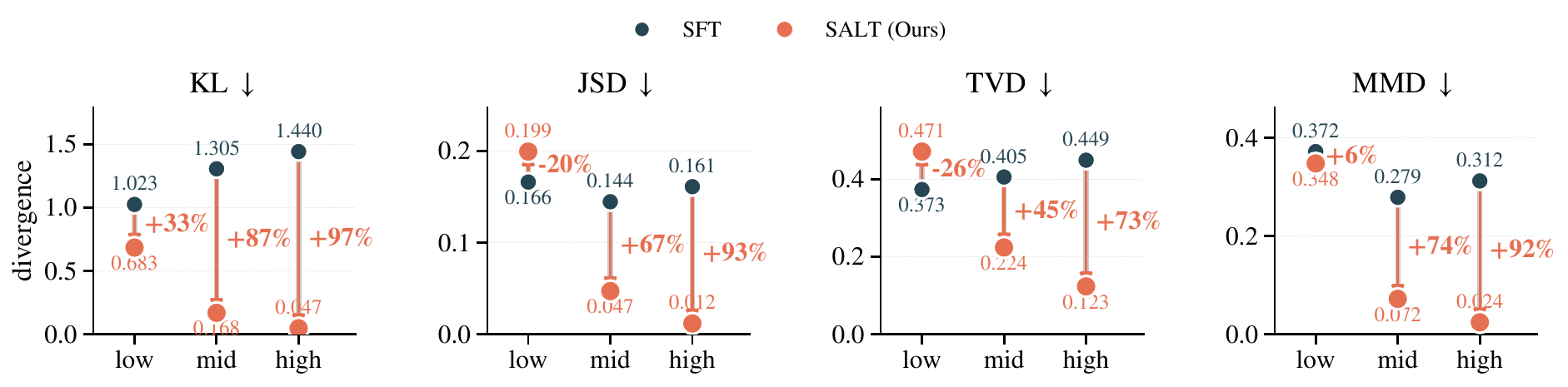}
\caption{Subjectivity-stratified comparison between SFT and SALT}
\end{subfigure}

\vspace{8pt}
\begin{subfigure}{\linewidth}
\centering
\includegraphics[width=\linewidth]{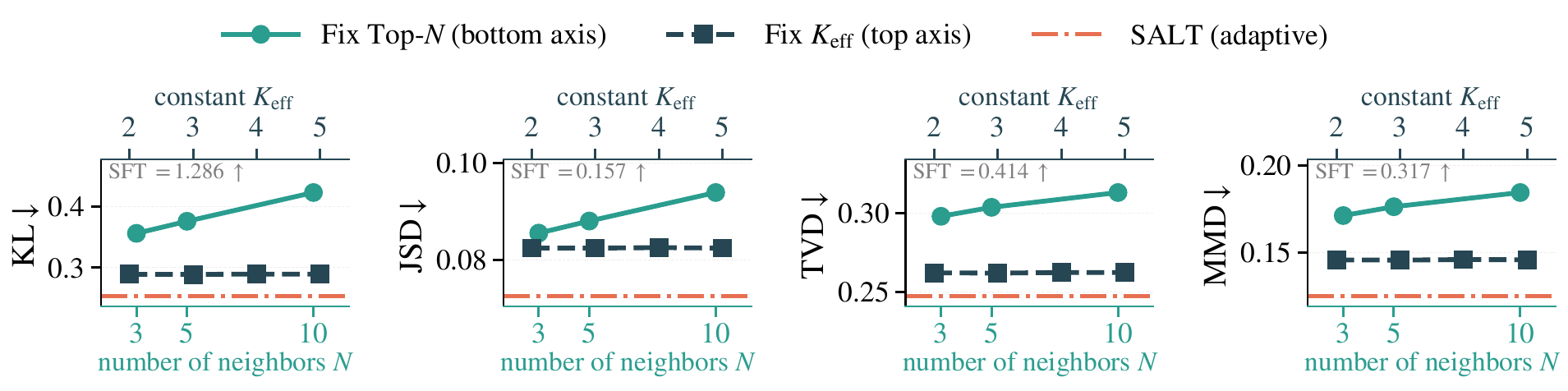}
\caption{Ablation over neighbor selection}
\end{subfigure}

\vspace{8pt}
\begin{subfigure}{\linewidth}
\centering
\includegraphics[width=\linewidth]{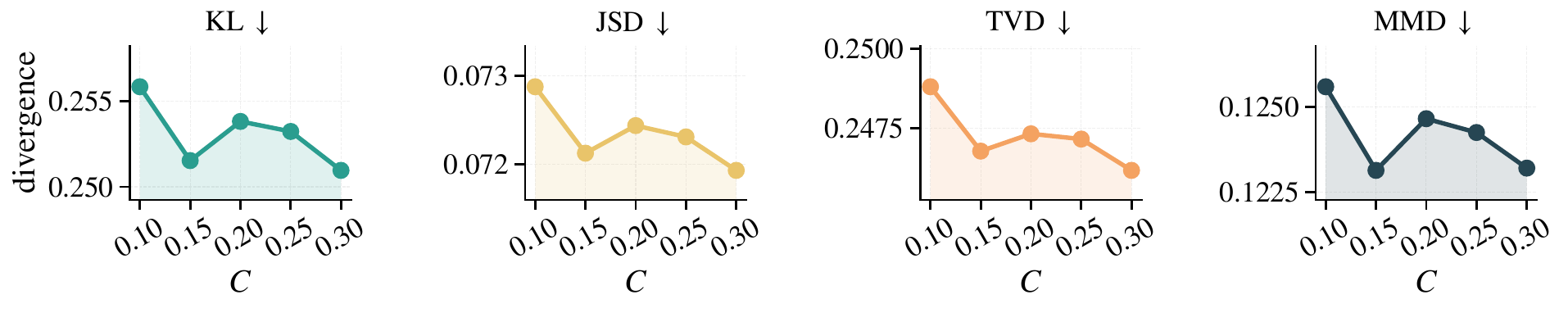}
\caption{Sensitivity to the radius scale $C$}
\end{subfigure}
\caption{Per-domain results on the \textbf{Political} domain, mirroring
\Cref{fig:subjectivity-levels}, \Cref{fig:ablation}, and
\Cref{fig:sensitivity-C} of the main text. All metrics are divergences
(lower is better).}
\label{fig:domain-political}
\end{figure}

\begin{figure}[p]
\centering
\begin{subfigure}{\linewidth}
\centering
\includegraphics[width=\linewidth]{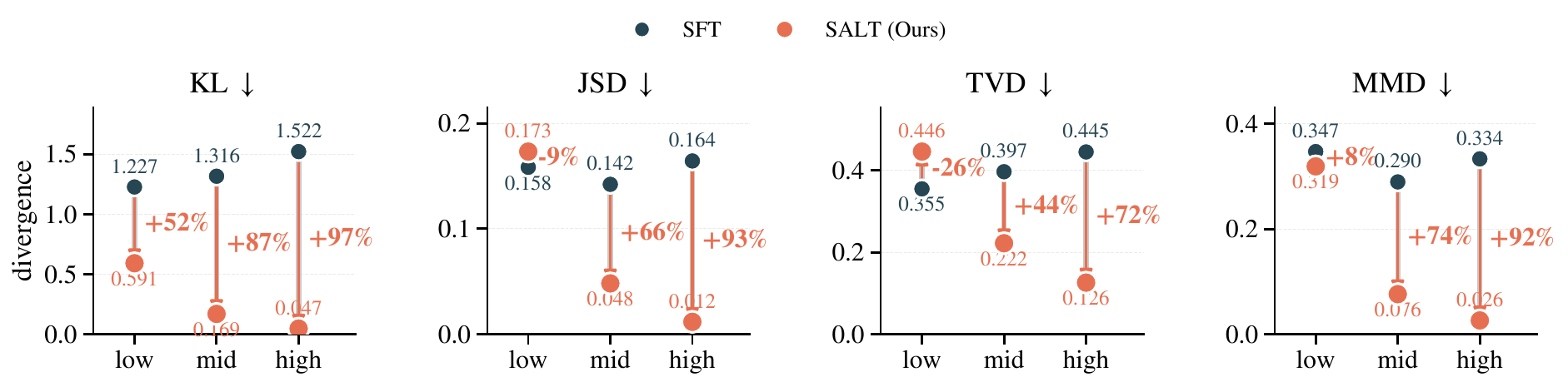}
\caption{Subjectivity-stratified comparison between SFT and SALT}
\end{subfigure}

\vspace{8pt}
\begin{subfigure}{\linewidth}
\centering
\includegraphics[width=\linewidth]{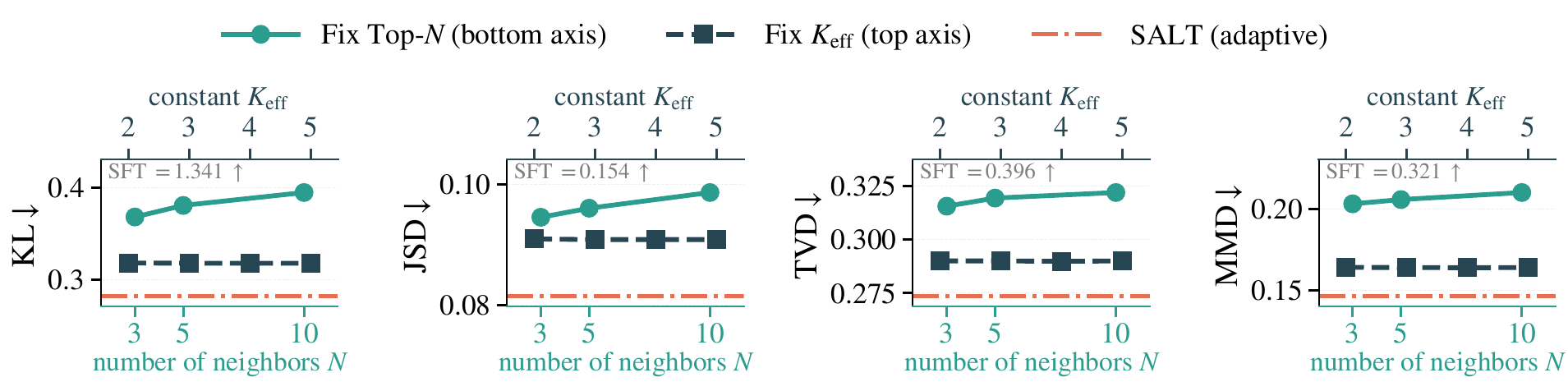}
\caption{Ablation over neighbor selection}
\end{subfigure}

\vspace{8pt}
\begin{subfigure}{\linewidth}
\centering
\includegraphics[width=\linewidth]{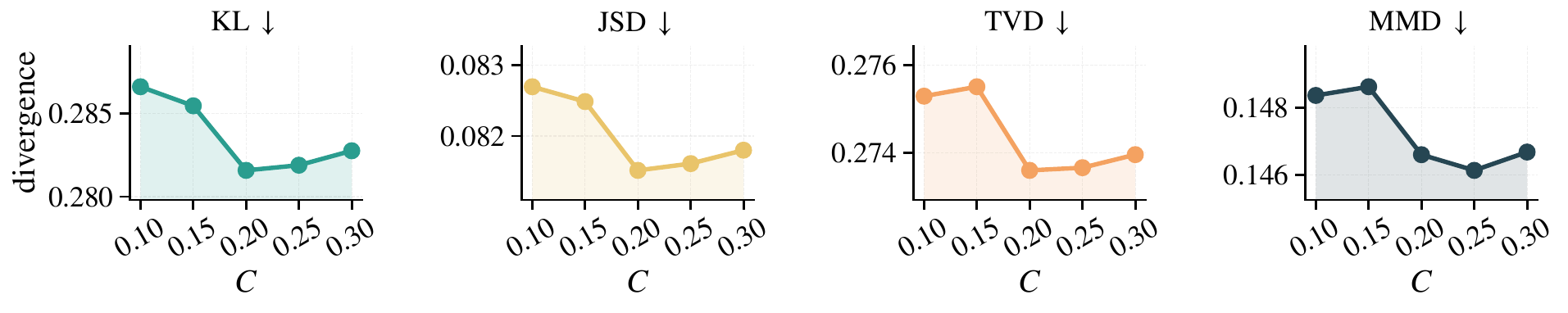}
\caption{Sensitivity to the radius scale $C$}
\end{subfigure}
\caption{Per-domain results on the \textbf{Technology} domain, mirroring
\Cref{fig:subjectivity-levels}, \Cref{fig:ablation}, and
\Cref{fig:sensitivity-C} of the main text. All metrics are divergences
(lower is better).}
\label{fig:domain-technology}
\end{figure}

\begin{figure}[p]
\centering
\begin{subfigure}{\linewidth}
\centering
\includegraphics[width=\linewidth]{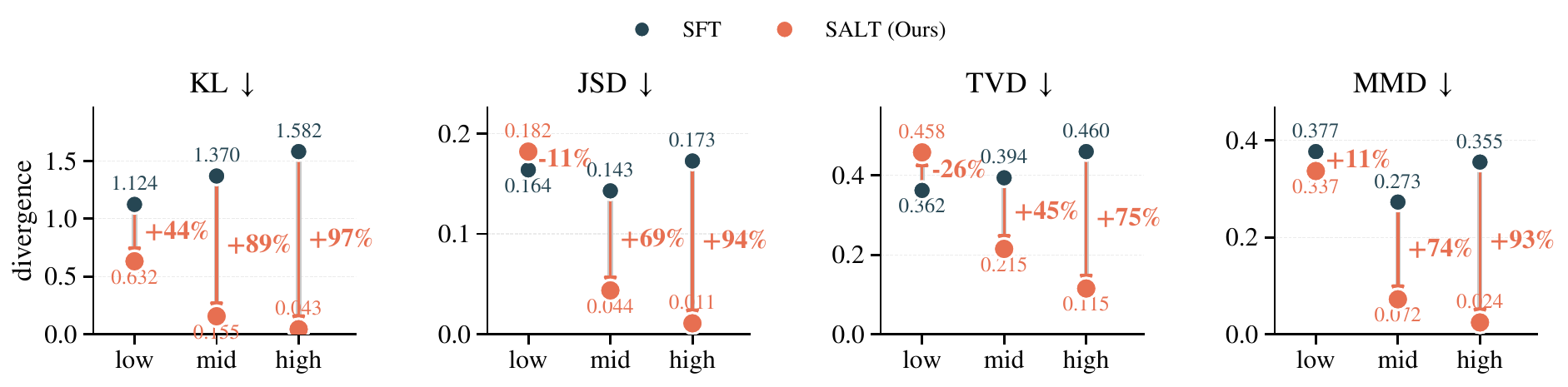}
\caption{Subjectivity-stratified comparison between SFT and SALT}
\end{subfigure}

\vspace{8pt}
\begin{subfigure}{\linewidth}
\centering
\includegraphics[width=\linewidth]{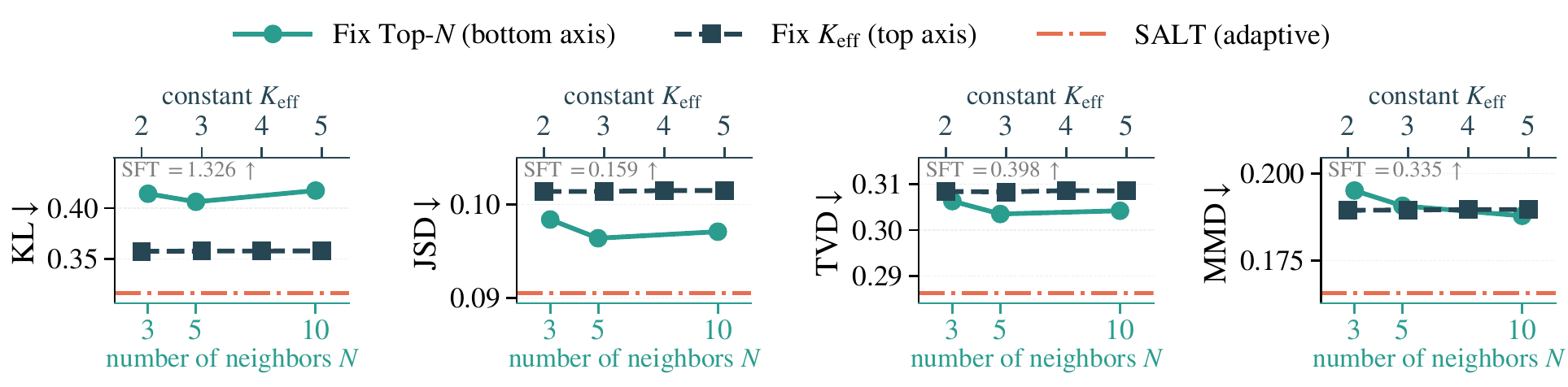}
\caption{Ablation over neighbor selection}
\end{subfigure}

\vspace{8pt}
\begin{subfigure}{\linewidth}
\centering
\includegraphics[width=\linewidth]{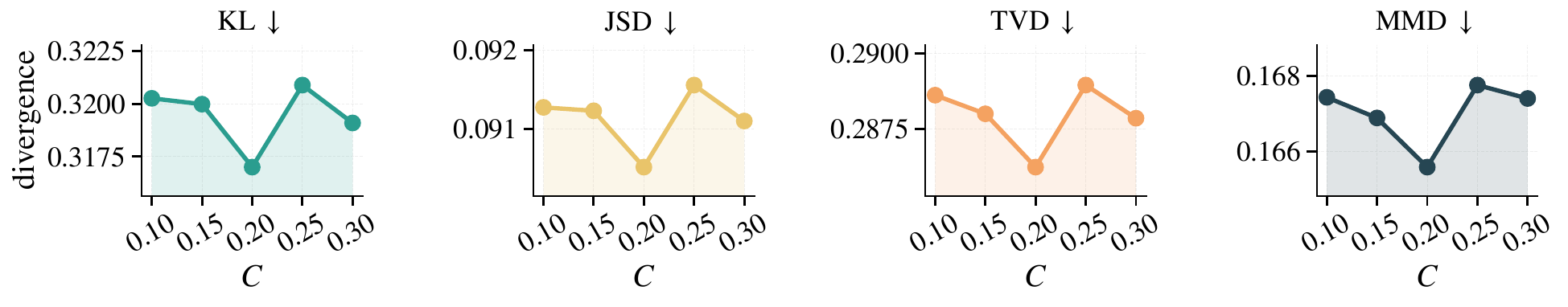}
\caption{Sensitivity to the radius scale $C$}
\end{subfigure}
\caption{Per-domain results on the \textbf{Social} domain, mirroring
\Cref{fig:subjectivity-levels}, \Cref{fig:ablation}, and
\Cref{fig:sensitivity-C} of the main text. All metrics are divergences
(lower is better).}
\label{fig:domain-social}
\end{figure}

\begin{figure}[p]
\centering
\begin{subfigure}{\linewidth}
\centering
\includegraphics[width=\linewidth]{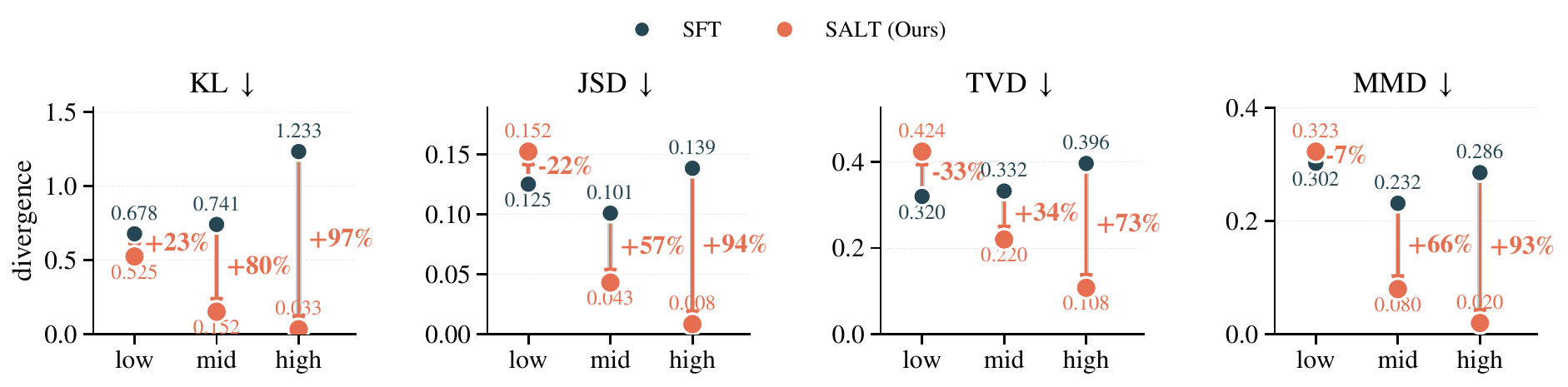}
\caption{Subjectivity-stratified comparison between SFT and SALT}
\end{subfigure}

\vspace{8pt}
\begin{subfigure}{\linewidth}
\centering
\includegraphics[width=\linewidth]{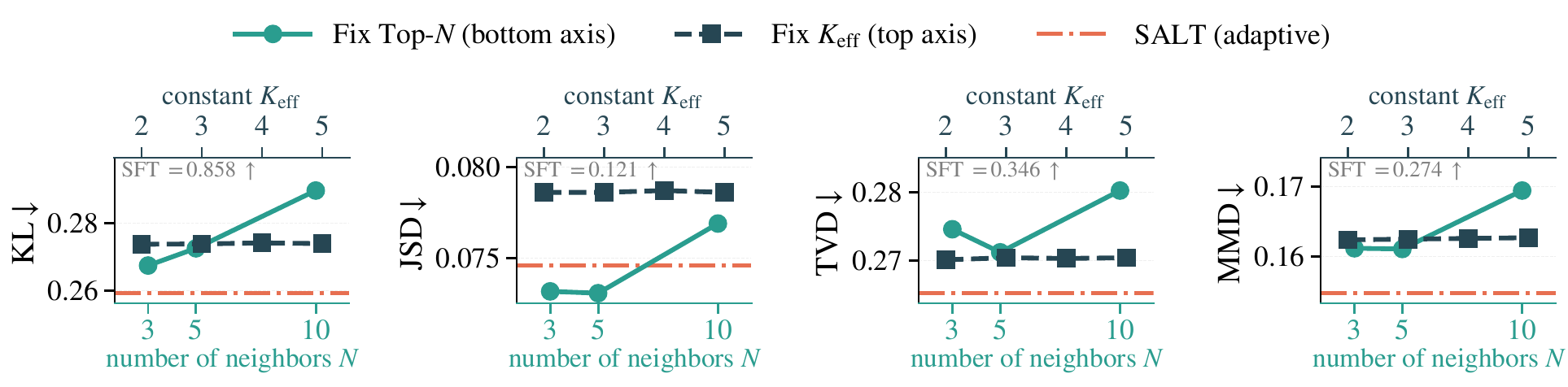}
\caption{Ablation over neighbor selection}
\end{subfigure}

\vspace{8pt}
\begin{subfigure}{\linewidth}
\centering
\includegraphics[width=\linewidth]{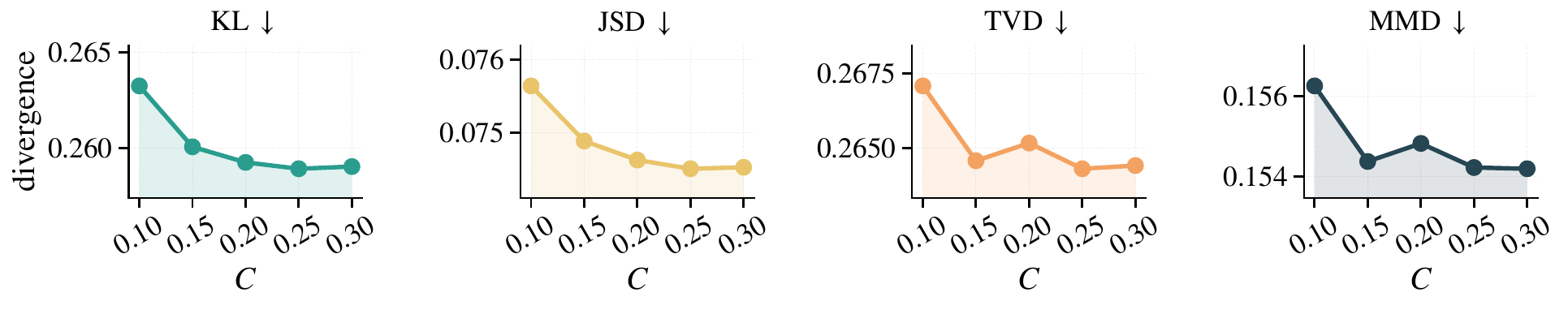}
\caption{Sensitivity to the radius scale $C$}
\end{subfigure}
\caption{Per-domain results on the \textbf{Culture} domain, mirroring
\Cref{fig:subjectivity-levels}, \Cref{fig:ablation}, and
\Cref{fig:sensitivity-C} of the main text. All metrics are divergences
(lower is better).}
\label{fig:domain-culture}
\end{figure}

\begin{figure}[p]
\centering
\begin{subfigure}{\linewidth}
\centering
\includegraphics[width=\linewidth]{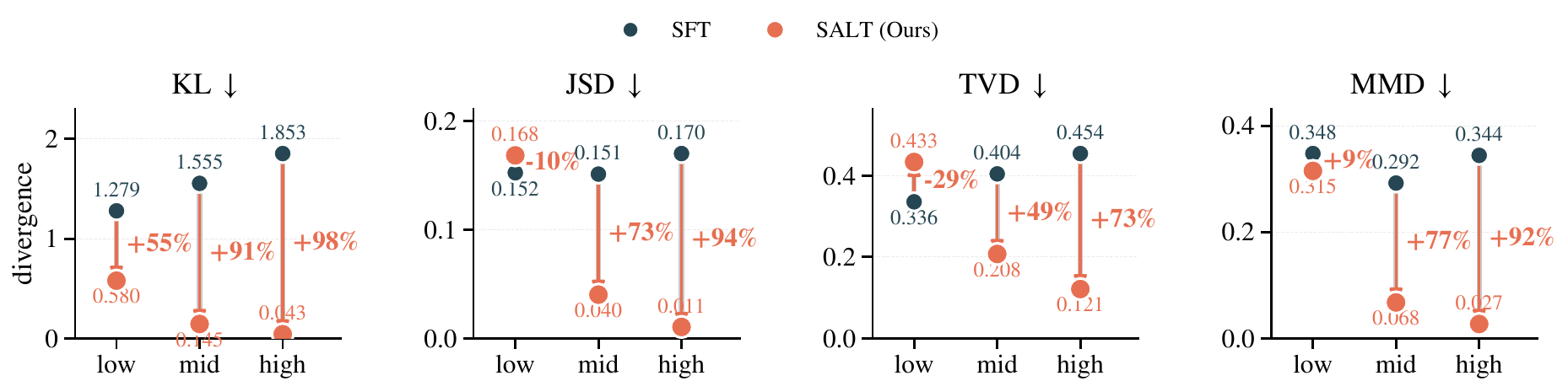}
\caption{Subjectivity-stratified comparison between SFT and SALT}
\end{subfigure}

\vspace{8pt}
\begin{subfigure}{\linewidth}
\centering
\includegraphics[width=\linewidth]{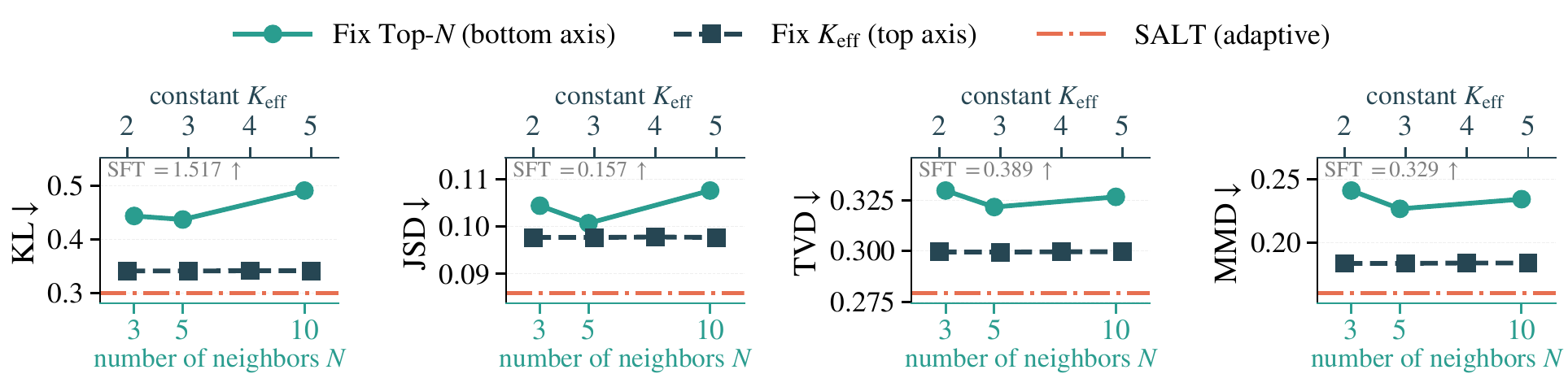}
\caption{Ablation over neighbor selection}
\end{subfigure}

\vspace{8pt}
\begin{subfigure}{\linewidth}
\centering
\includegraphics[width=\linewidth]{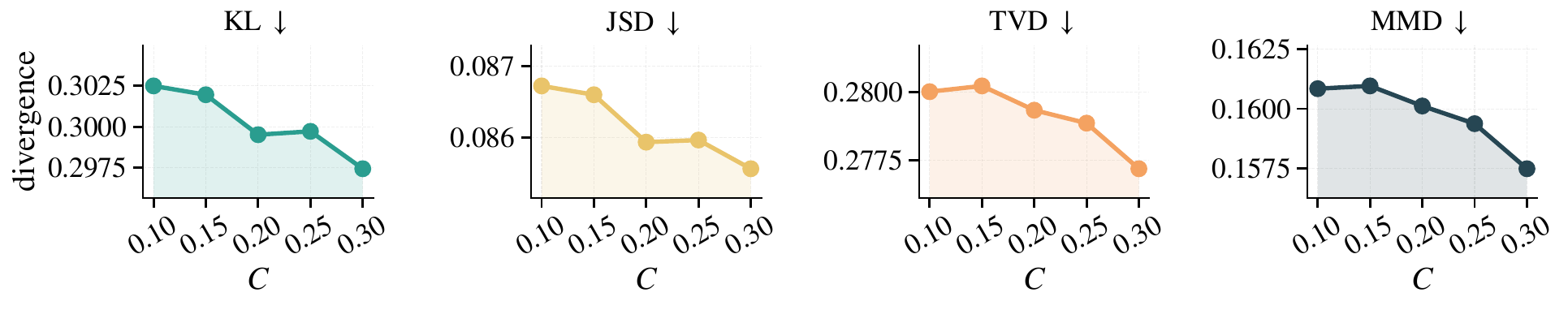}
\caption{Sensitivity to the radius scale $C$}
\end{subfigure}
\caption{Per-domain results on the \textbf{Health} domain, mirroring
\Cref{fig:subjectivity-levels}, \Cref{fig:ablation}, and
\Cref{fig:sensitivity-C} of the main text. All metrics are divergences
(lower is better).}
\label{fig:domain-health}
\end{figure}

\begin{figure}[p]
\centering
\begin{subfigure}{\linewidth}
\centering
\includegraphics[width=\linewidth]{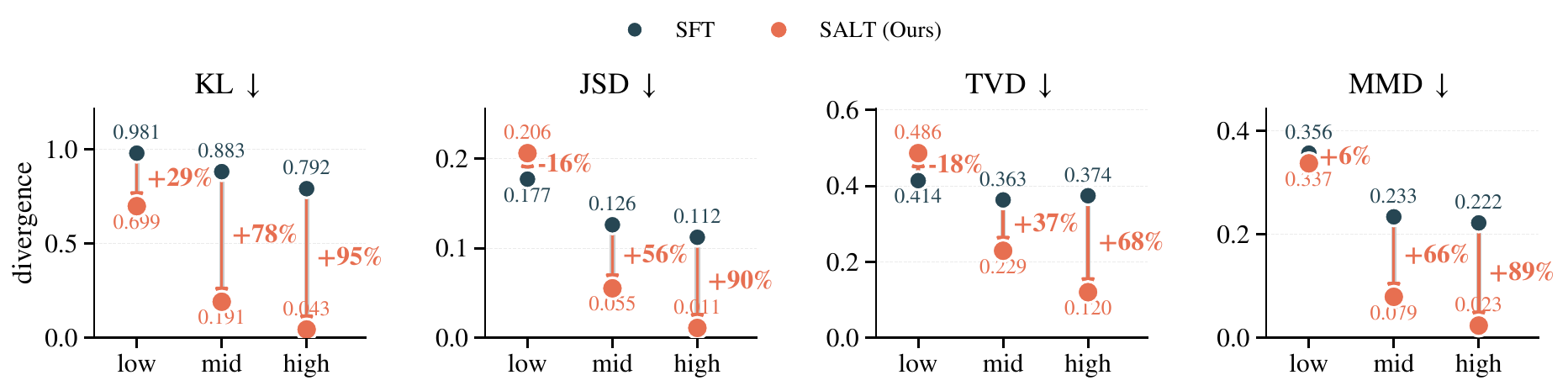}
\caption{Subjectivity-stratified comparison between SFT and SALT}
\end{subfigure}

\vspace{8pt}
\begin{subfigure}{\linewidth}
\centering
\includegraphics[width=\linewidth]{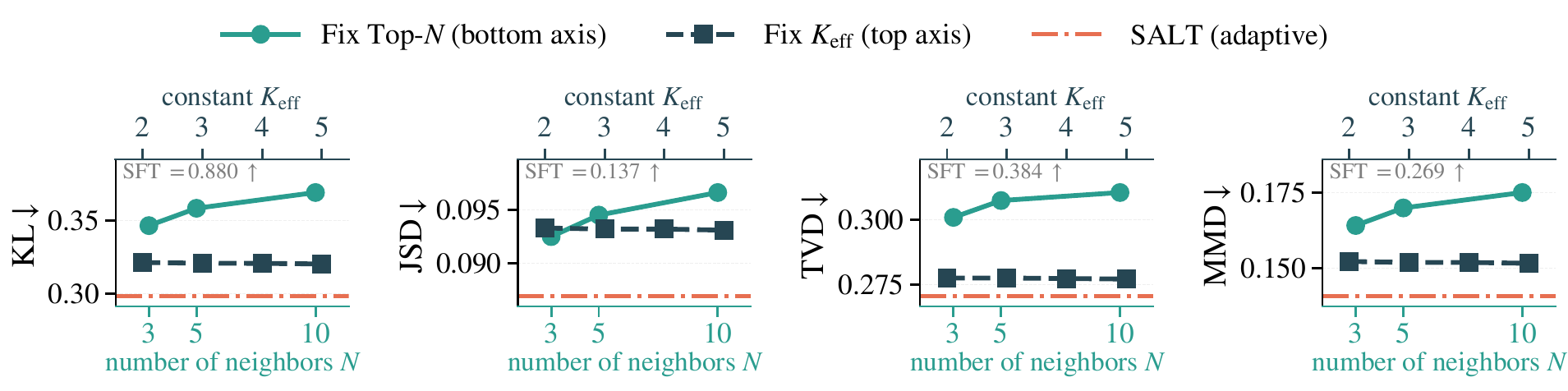}
\caption{Ablation over neighbor selection}
\end{subfigure}

\vspace{8pt}
\begin{subfigure}{\linewidth}
\centering
\includegraphics[width=\linewidth]{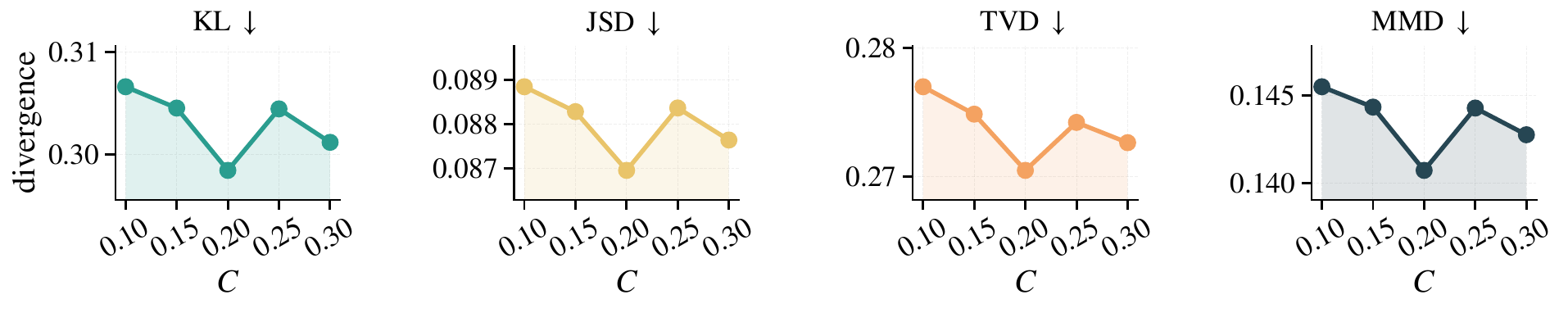}
\caption{Sensitivity to the radius scale $C$}
\end{subfigure}
\caption{Per-domain results on the \textbf{Environment} domain, mirroring
\Cref{fig:subjectivity-levels}, \Cref{fig:ablation}, and
\Cref{fig:sensitivity-C} of the main text. All metrics are divergences
(lower is better).}
\label{fig:domain-environment}
\end{figure}

\begin{figure}[p]
\centering
\begin{subfigure}{\linewidth}
\centering
\includegraphics[width=\linewidth]{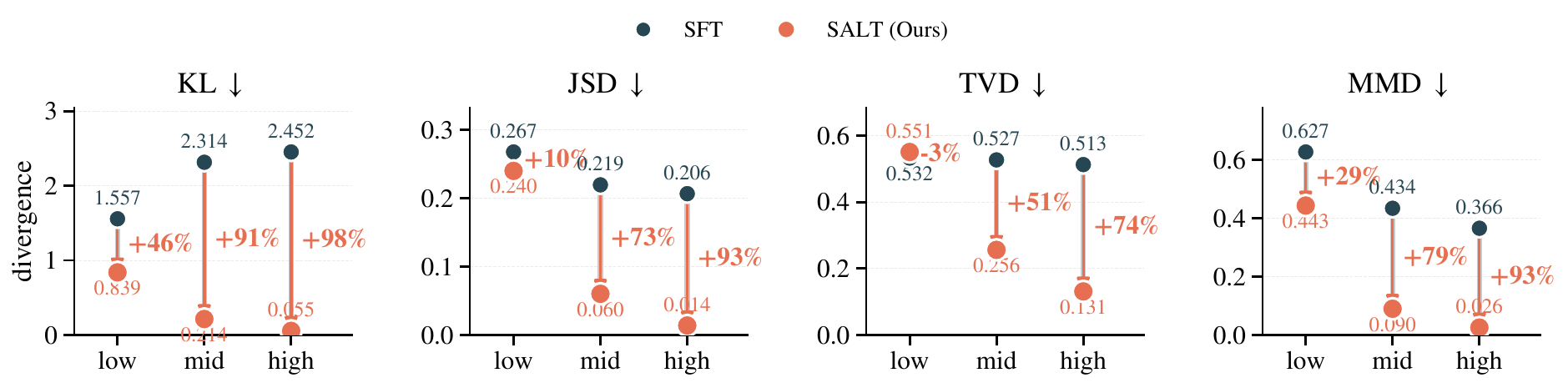}
\caption{Subjectivity-stratified comparison between SFT and SALT}
\end{subfigure}

\vspace{8pt}
\begin{subfigure}{\linewidth}
\centering
\includegraphics[width=\linewidth]{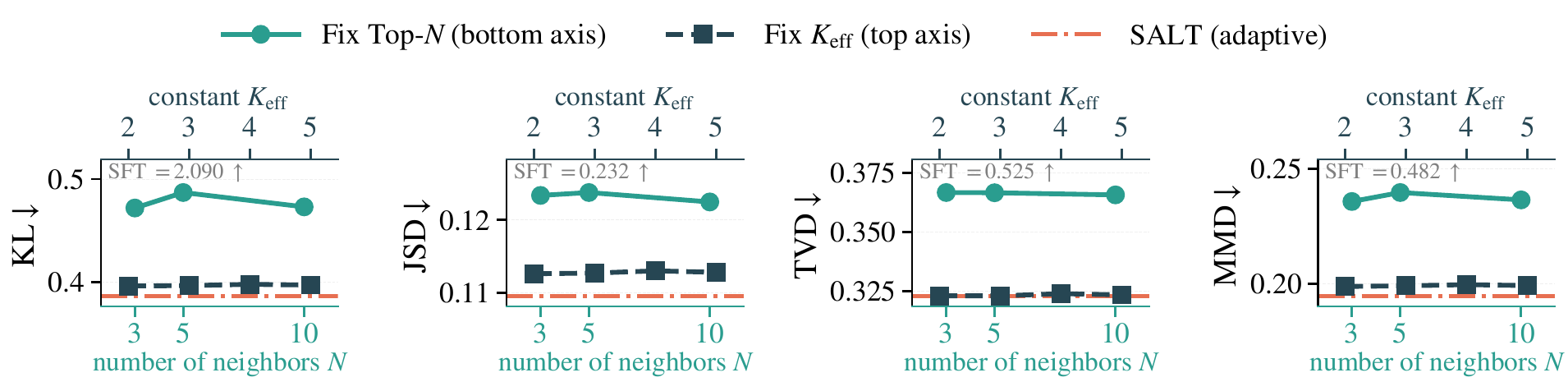}
\caption{Ablation over neighbor selection}
\end{subfigure}

\vspace{8pt}
\begin{subfigure}{\linewidth}
\centering
\includegraphics[width=\linewidth]{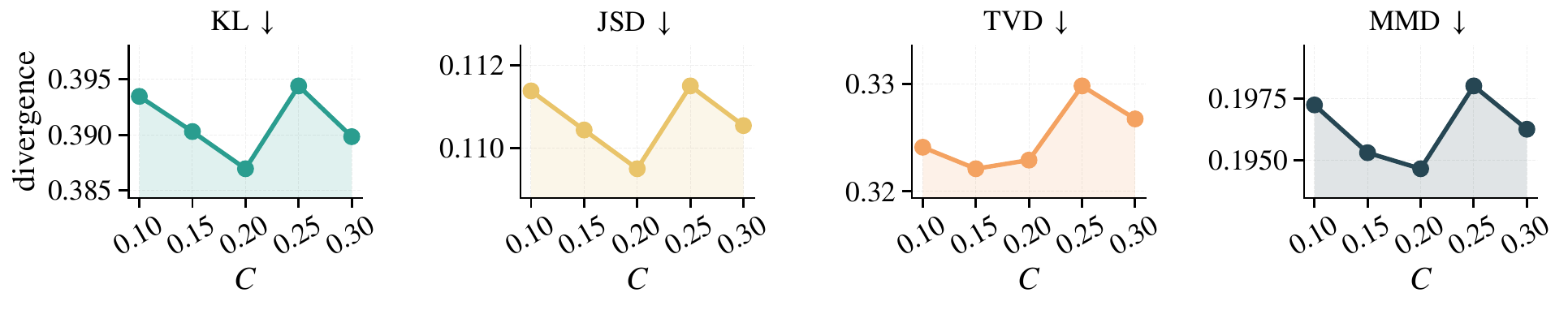}
\caption{Sensitivity to the radius scale $C$}
\end{subfigure}
\caption{Per-domain results on the \textbf{Education} domain, mirroring
\Cref{fig:subjectivity-levels}, \Cref{fig:ablation}, and
\Cref{fig:sensitivity-C} of the main text. All metrics are divergences
(lower is better).}
\label{fig:domain-education}
\end{figure}

\section{Statistical Significance}
\label{app:significance}

Evaluation is deterministic (\Cref{subsec:exp-setup}), so the only
source of uncertainty in the reported metrics is the finiteness of the
test set. We therefore quantify it with a paired bootstrap over test
contexts: we resample the 2,800 test contexts with replacement 1,000
times, apply the same resampled index set to every method, and recompute
each metric on every replicate. \Cref{tab:bootstrap-ci} reports the
point estimates with 95\% percentile intervals. The intervals of SALT
and every baseline are disjoint on every metric; testing the paired
differences directly, the 95\% interval of SALT minus each comparator
excludes zero on all four metrics, including the strongest ablation
variant (fixed $K_{\mathrm{eff}}$: KL difference $-0.0273$,
CI $[-0.0307, -0.0237]$). SALT's improvements are therefore
statistically significant rather than an artifact of the test split.

\begin{table}[t]
\centering
\caption{Point estimates and 95\% paired-bootstrap confidence intervals
on the full test set (1,000 resamples over contexts). All metrics are
divergences (lower is better).}
\label{tab:bootstrap-ci}
\small
\renewcommand{\arraystretch}{0.95}
\setlength{\tabcolsep}{3.5pt}
\begin{tabular}{@{}lcccc@{}}
\toprule
\textbf{Method} & \textbf{KL} & \textbf{JSD} & \textbf{TVD} & \textbf{MMD} \\
\midrule
Pretrained & $4.2801\ [4.1628, 4.3949]$ & $0.2851\ [0.2785, 0.2912]$ & $0.5745\ [0.5646, 0.5833]$ & $0.6350\ [0.6161, 0.6524]$ \\
SFT & $1.2871\ [1.2281, 1.3439]$ & $0.1524\ [0.1474, 0.1573]$ & $0.3953\ [0.3869, 0.4034]$ & $0.3194\ [0.3060, 0.3324]$ \\
DPO & $6.8819\ [6.7226, 7.0539]$ & $0.2690\ [0.2625, 0.2752]$ & $0.5497\ [0.5396, 0.5590]$ & $0.5774\ [0.5595, 0.5935]$ \\
PPO & $4.1393\ [4.0256, 4.2464]$ & $0.2588\ [0.2524, 0.2647]$ & $0.5403\ [0.5302, 0.5490]$ & $0.5578\ [0.5407, 0.5743]$ \\
DSA & $3.3904\ [3.2932, 3.4851]$ & $0.2343\ [0.2284, 0.2403]$ & $0.4937\ [0.4857, 0.5019]$ & $0.4418\ [0.4281, 0.4567]$ \\
\addlinespace[1pt]
SALT (Ours) & $\mathbf{0.2880}\ [0.2749, 0.3009]$ & $\mathbf{0.0825}\ [0.0788, 0.0863]$ & $\mathbf{0.2720}\ [0.2653, 0.2787]$ & $\mathbf{0.1510}\ [0.1440, 0.1582]$ \\
\midrule
Top-$N$ ($N{=}3$) & $0.3653\ [0.3497, 0.3809]$ & $0.0925\ [0.0886, 0.0964]$ & $0.3078\ [0.3011, 0.3144]$ & $0.1921\ [0.1833, 0.2014]$ \\
Fixed $K_{\mathrm{eff}}$ ($K{=}2$) & $0.3152\ [0.3017, 0.3290]$ & $0.0900\ [0.0861, 0.0939]$ & $0.2839\ [0.2769, 0.2909]$ & $0.1656\ [0.1586, 0.1734]$ \\
Global-Freq & $0.3879\ [0.3713, 0.4036]$ & $0.0933\ [0.0894, 0.0970]$ & $0.3065\ [0.2993, 0.3130]$ & $0.1886\ [0.1798, 0.1973]$ \\
\bottomrule
\end{tabular}
\end{table}

\FloatBarrier

\section{SubjSim Dataset Details}
\label{sec:dataset}

\paragraph{Why elicited propensities.}
Exact repeated measurement of the same person is not a viable route to
the distributional ground truth, since repetition can change memory,
reflection, fatigue, and demand effects. \textsc{SubjSim} therefore
operationalizes the target via elicited subjective response
propensities: respondents allocate plausibility across the same action
space, so the hidden distributional target is available for evaluation
while training methods receive only a single hard action.

\begin{figure*}[!t]
    \centering
    \includegraphics[width=\textwidth]{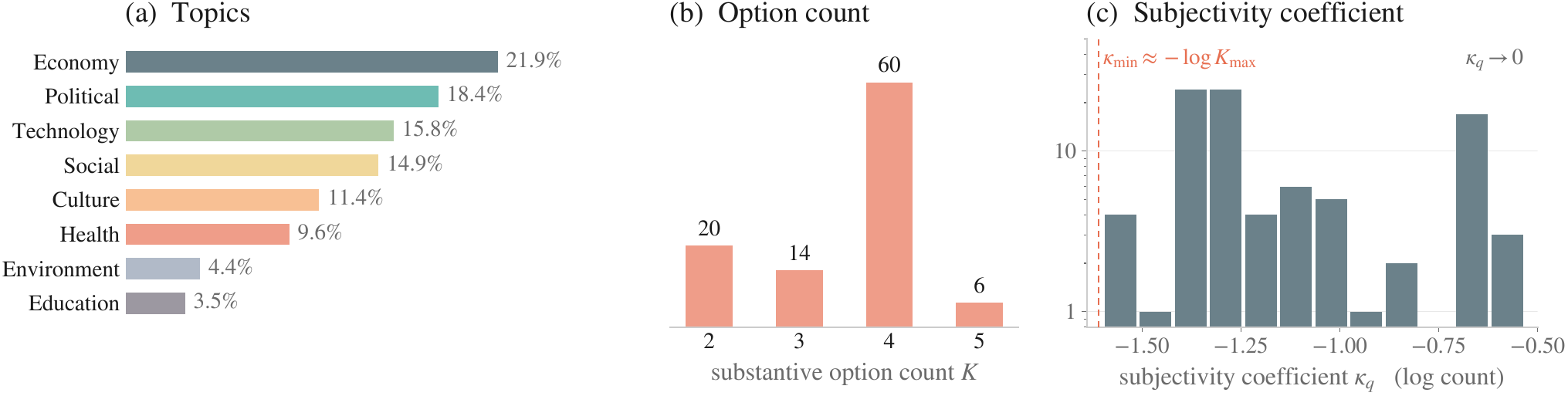}
    \caption{Overview of the SubjSim dataset. (a) Topics. (b) Option count $K$. (c) Subjectivity coefficient.}
    \label{fig:dataset_stats}
\end{figure*}

\subsection{Question Pool and Filtering}

\paragraph{Source Surveys.}
Questions are drawn from seven internationally standardized social survey
programs: the American Trends Panel (ATP), General Social Survey (GSS),
World Values Survey (WVS), American National Election Studies (ANES),
Chinese General Social Survey (CGSS), European Social Survey (ESS), and
International Social Survey Programme (ISSP). These programs were
selected for their broad topical coverage, institutional authority, and
diversity of question types. The initial candidate pool contains 5,964
items in total.
\paragraph{Translation.} All non-Chinese items are translated into Chinese using DeepSeek-chat \citep{liu2024deepseekv3}. CGSS items are retained in their original Chinese form.

\paragraph{Automated Screening.} Each candidate item is evaluated by DeepSeek-chat \citep{liu2024deepseekv3} along three dimensions:

\begin{itemize}
    \item \textbf{Cultural suitability:} whether the item is appropriate for Chinese respondents, considering cultural sensitivity, social norms, privacy boundaries, linguistic conventions, and regional variation.
    \item \textbf{Option ordinality:} whether the response options follow a logical order (ordinal) or are unordered categories (nominal). Only nominal items are retained, as ordinal scales introduce additional measurement assumptions that complicate distributional evaluation with metrics such as TVD and JSD.
    \item \textbf{Question objectivity:} whether the item has a factually verifiable answer (objective) or depends on personal attitudes, feelings, or preferences (subjective).
\end{itemize}

The system prompt used for automated screening is as follows:

\begin{tcolorbox}[colback=gray!10, colframe=gray!50, breakable, title=Screening Prompt]
\small\ttfamily
You are a professional survey design expert. Please perform a three-dimensional annotation analysis on the given survey question.

\textbf{Annotation Dimensions}

\textbf{1. Cultural Suitability}
\begin{itemize}
    \item Determine whether the question is appropriate for Chinese respondents.
    \item Consider: cultural sensitivity, social norms, privacy boundaries, linguistic conventions, and regional variation.
    \item Output: suitable / caution / unsuitable, with a brief explanation.
\end{itemize}

\textbf{2. Option Ordinality}
\begin{itemize}
    \item Determine whether the response options have a logical order.
    \item Ordinal: options exhibit a clear gradient, ranking, or sequence (e.g., very dissatisfied $\rightarrow$ very satisfied; 18--25 $\rightarrow$ 26--35 $\rightarrow$ 36+).
    \item Nominal: options are parallel with no inherent order (e.g., red/blue/green; football/basketball/swimming).
    \item Output: ordinal / nominal, with justification.
\end{itemize}

\textbf{3. Question Objectivity}
\begin{itemize}
    \item Determine whether the question is objective or subjective.
    \item Objective: the answer is factually verifiable and does not depend on personal feelings or opinions (e.g., age, household size, education level, occupation, home ownership, weekly exercise frequency).
    \item Subjective: the answer depends on personal attitudes, feelings, evaluations, or preferences (e.g., satisfaction, sense of identity, importance ratings, willingness, brand preference).
    \item Output: objective / subjective, with justification.
\end{itemize}
\end{tcolorbox}

After automated screening, 109 items are labeled as suitable, nominal,
and subjective. These items proceed to manual review. The remaining
1,012 suitable objective items serve as the pool for demographic
attribute dimensions.

\paragraph{Manual Review and Localization.}
We manually review the 109 subjective items, remove duplicates, verify
subjectivity, and adapt phrasing to the Chinese cultural context. This
yields the final 100 questions.

\paragraph{Demographic Attribute Dimensions.}
From the 1,012 objective items, we manually select those most relevant to
the 100 subjective questions, yielding 30 demographic attribute
dimensions.

\begin{table}[t]
\centering
\caption{Source distribution of the 100 subjective questions in SubjSim.}
\begin{tabular}{@{}lrr@{}}
\toprule
\textbf{Source} & \textbf{Count} & \textbf{Proportion} \\
\midrule
ATP  & 35 & 35.0\% \\
GSS  & 19 & 19.0\% \\
WVS  & 14 & 14.0\% \\
CGSS &  9 &  9.0\% \\
ANES &  9 &  9.0\% \\
ESS  &  9 &  9.0\% \\
ISSP &  5 &  5.0\% \\
\midrule
\textbf{Total} & \textbf{100} & \textbf{100\%} \\
\bottomrule
\end{tabular}
\end{table}

\FloatBarrier

\subsection{Question and Demographic Attribute Examples}

\begin{table*}[!t]
\centering
\caption{Examples of Demographic Attributes and Subjective Questions}
\label{app:full_examples}
\small
\renewcommand{\arraystretch}{1.15}
\setlength{\tabcolsep}{4pt}
\begin{tabularx}{\textwidth}{@{}p{1.9cm}XX@{}}
\toprule
\textbf{Domain} & \textbf{Question} & \textbf{Response Options} \\
\midrule
\rowcolor{gray!15}
\multicolumn{3}{@{}l}{\textbf{Demographic Attributes}} \\
\midrule
\multirow{3}{*}{Demographic} 
& What is your current academic year? & 1.~Freshman/Sophomore; 2.~Junior/Senior; 3.~Master's; 4.~PhD; 5.~Not~a~student \\
\cmidrule{2-3}
& Are you an only child? & 1.~Yes; 2.~No \\
\cmidrule{2-3}
& What is your current employment status? & 1.~Full-time; 2.~Part-time; 3.~Self-employed; 4.~Retired; 5.~Homemaker; 6.~Student; 7.~Unemployed; 8.~Other \\
\midrule
\rowcolor{gray!15}
\multicolumn{3}{@{}l}{\textbf{Subjective Questions}} \\
\midrule
\multirow{2}{*}{\parbox{1.5cm}{Political\\System}} 
& Over the next 30 years, which social trend concerns you the most? & 1.~AI replacing jobs; 2.~Social stratification; 3.~Misinformation; 4.~Weakening family structures \\
\cmidrule{2-3}
& Among occupational groups, which do you trust the most? & 1.~Healthcare/education; 2.~Law enforcement; 3.~Business; 4.~Non-profits \\
\midrule
\multirow{2}{*}{Education} 
& Who should ensure young people acquire skills for good jobs? & 1.~Government; 2.~Employers; 3.~Education system; 4.~Individuals \\
\cmidrule{2-3}
& Which quality is most important for children to learn? & 1.~Socially adept; 2.~Obedient; 3.~Hard work; 4.~Helpful; 5.~Independent thinking \\
\midrule
\multirow{2}{*}{\parbox{1.5cm}{Social\\Relations}} 
& What do you most often do in free time? & 1.~Social entertainment; 2.~Leisure; 3.~Self-improvement; 4.~Exercise \\
\cmidrule{2-3}
& Which is more important: considerate or proper behavior? & 1.~Considerate; 2.~Proper \\
\midrule
\multirow{2}{*}{\parbox{1.5cm}{Health \\Well-being}} 
& What is your view on vaccination? & 1.~Mandatory; 2.~Personal choice; 3.~Cautious; 4.~Not mandatory \\
\cmidrule{2-3}
& What is the biggest problem with the healthcare system? & 1.~Over-prescription; 2.~High cost; 3.~Drug safety; 4.~Uneven resources \\
\midrule
\multirow{2}{*}{\parbox{1.5cm}{Economy \\Labor}} 
& Do you prefer full-time employment? & 1.~Yes; 2.~No \\
\cmidrule{2-3}
& What is the biggest challenge if changing jobs? & 1.~Salary/benefits; 2.~Skill competitiveness; 3.~Few opportunities; 4.~Not difficult \\
\midrule
\multirow{2}{*}{\parbox{1.5cm}{Values \\Culture}} 
& Do you believe in life after death? & 1.~Yes; 2.~No \\
\cmidrule{2-3}
& Which value should society prioritize? & 1.~Equal opportunity; 2.~Individual freedom; 3.~Social order; 4.~Tradition \\
\midrule
\multirow{2}{*}{\parbox{1.5cm}{Technology\\\& Society}} 
& What is your view on genetically modified foods? & 1.~Healthier; 2.~More harmful; 3.~No difference \\
\cmidrule{2-3}
& Who should protect personal information online? & 1.~Companies; 2.~Individuals; 3.~Public institutions \\
\midrule
\multirow{2}{*}{\parbox{1.8cm}{Environment\\\& Energy}} 
& Which position do you lean toward on climate change? & 1.~Existential crisis; 2.~Politicized; 3.~Long-term issue; 4.~Natural cycles \\
\cmidrule{2-3}
& Regarding energy, what are you most concerned about? & 1.~Prices; 2.~Outages; 3.~Fossil fuel reliance; 4.~Natural disasters \\
\bottomrule
\end{tabularx}
\end{table*}

\FloatBarrier

\subsection{Annotator Demographics}
\label{subsec:annotator}
We recruited 193 annotators via a university online forum, including
both students and non-students. Figure~\ref{fig:demographics} summarizes
their distributions across age, gender, academic status, and field of
study. The pool skews toward young, university-educated individuals, with
a roughly balanced gender ratio and a mix of STEM and Humanities
backgrounds. This limits claims about population-level
representativeness, but it is less central to the benchmark's main
target: evaluating distributional alignment for individual behavioral
tendencies.

\begin{figure*}[!t]
\centering
\includegraphics[width=\textwidth]{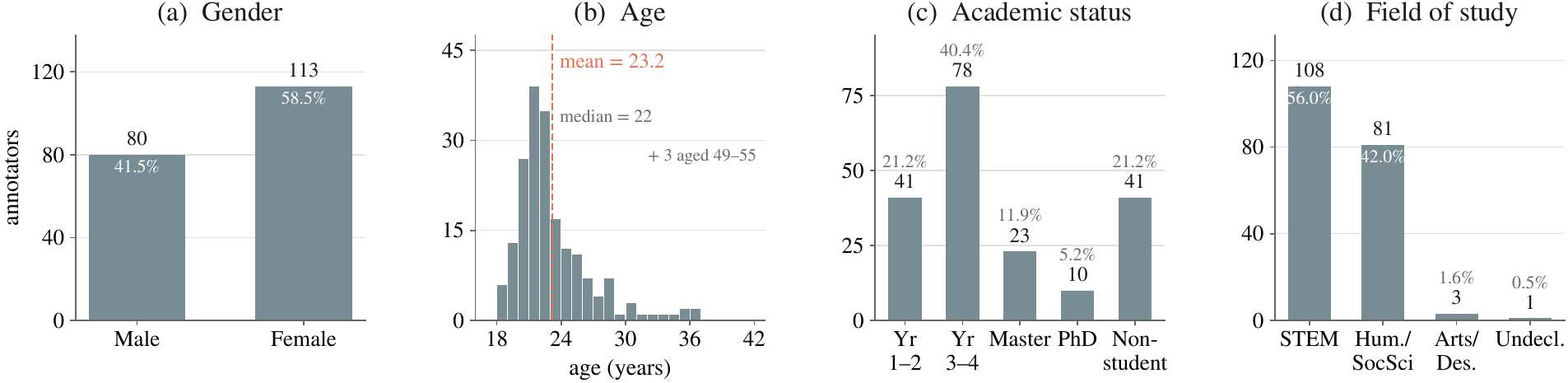}
\caption{Demographic distributions of annotators across gender, age, academic status, and field of study.}
\label{fig:demographics}
\end{figure*}

\FloatBarrier

\subsection{Probability Ball Allocation Protocol}
\label{subsec:ball_allocation}
\begin{figure*}[!t]
\centering
\includegraphics[width=0.85\textwidth]{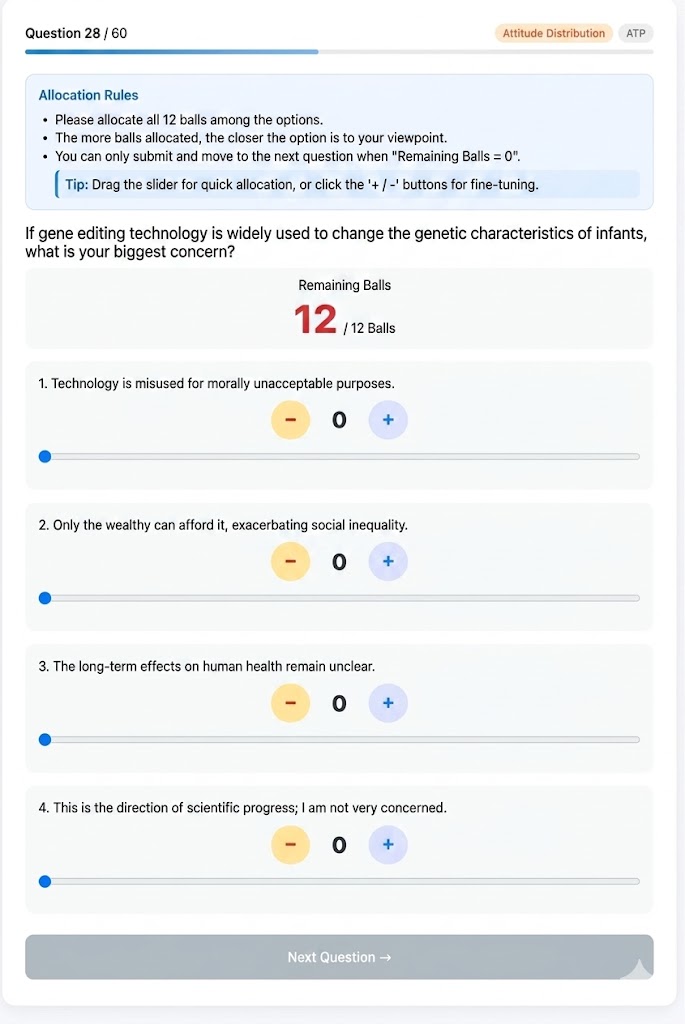}
\caption{Screenshot of the annotation interface. The platform was deployed in Chinese for native Chinese-speaking annotators; the interface shown here is an English translation for presentation purposes.}
\label{fig:interface}
\end{figure*}

In the probability-ball protocol, annotators distribute a fixed number of
balls across all available options and must allocate all balls before
submission. The fraction assigned to each option represents the
annotator's subjective probability for that option. This avoids known
issues with Likert-scale ratings such as scale bias and cross-item
incomparability.

The number of balls scales with $K$ to balance resolution and cognitive load: 10 balls for $K=2$, 12 for $K=3$ and $K=4$, 15 for $K=5$, and 20 otherwise.

We implemented a web-based annotation platform where annotators adjust
ball counts per option via sliders or $+/-$ buttons. The remaining ball
count is displayed in real time, and submission is blocked until all
balls are allocated. Figure~\ref{fig:interface} shows a screenshot of
the interface.

\FloatBarrier

\end{document}